\documentclass[journal]{IEEEtran}
\ifCLASSINFOpdf
\else
\fi

\usepackage{amsmath, amsthm, amssymb, bm, graphicx, hyperref, mathrsfs, color, algorithmicx, cite}
\usepackage[linesnumbered, ruled]{algorithm2e}
\usepackage{hyperref}
\usepackage{subfigure}
\usepackage{float}  
\newtheorem{theorem}{Theorem}
\newtheorem{lemma}{Lemma}
\newtheorem{corollary}{Corollary}

\begin{document}
%
\title{Federated Learning with Integrated Sensing, Communication, and Computation: Frameworks and Performance Analysis}
%
%

\author{Yipeng~Liang,~\IEEEmembership{Student~Member,~IEEE,}
        Qimei~Chen,~\IEEEmembership{Member,~IEEE,} 
        and~Hao~Jiang,~\IEEEmembership{Member,~IEEE}
\thanks{The authors are with the School of Electronic Information, Wuhan University, Wuhan 430072, China e-mail: (liangyipeng@whu.edu.cn, chenqimei@whu.edu.cn, jh@whu.edu.com).}
\thanks{This manuscript is a preliminary version of the work and may be subject to further revisions.}
}
\markboth{Journal of \LaTeX\ Class Files,~Vol.~14, No.~8, August~2015}%
{Shell \MakeLowercase{\textit{et al.}}: Bare Demo of IEEEtran.cls for IEEE Journals}
%



\maketitle

\begin{abstract}

With the emergence of integrated sensing, communication, and computation (ISCC) in the upcoming 6G era, federated learning with ISCC (FL-ISCC), integrating sample collection in the sensing process, local training in the computation process, and parameter exchange and aggregation in the communication process, has garnered increasing interest for enhancing training efficiency. Currently, FL-ISCC primarily includes two algorithms: FedAVG-ISCC and FedSGD-ISCC. However, the theoretical understanding of the performance and advantages of FedAVG-ISCC and FedSGD-ISCC remains limited. To address this gap, we investigate a general FL-ISCC framework for wireless networks, implementing both FedAVG-ISCC and FedSGD-ISCC. 
We experimentally demonstrate the substantial potential of the ISCC framework to enhance training efficiency in terms of latency and energy consumption in FL. Furthermore, we provide a theoretical analysis and comparison of the performance of FedAVG-ISCC and FedSGD-ISCC within this framework. The results reveal that: 
1) Both sample collection and communication errors adversely affect the performance of both algorithms, underscoring the need for careful design in practical applications to enhance FL-ISCC performance.
2) Under independent and identically distributed (IID) data, FedAVG-ISCC outperforms FedSGD-ISCC, due to the advantage of multiple local updates.
3) Under non-independent and identically distributed (Non-IID) data,  FedSGD-ISCC demonstrates greater robustness than FedAVG-ISCC. In particular, the local updates in FedAVG-ISCC amplify the impact of Non-IID data, resulting in significant performance degradation as the degree of Non-IID data increases. Conversely, FedSGD-ISCC, which does not incorporate multiple local updates, maintains performance levels comparable to those under IID conditions.
4) FedSGD-ISCC is more robust to communication errors compared to FedAVG-ISCC. Specifically, as communication errors increase, FedAVG-ISCC experiences significant performance deterioration, as its learning rate amplifies the impact of these errors, whereas FedSGD-ISCC remains largely unaffected.
 Extensive simulations verify the effectiveness of our considered FL-ISCC framework and validate the theoretical analysis presented in this work.


\end{abstract}


%
\IEEEpeerreviewmaketitle

\section{Introduction}

%
%
%
%
%
%


\IEEEPARstart{I}{t} is anticipated that 6G will extend beyond mobile internet to support ubiquitous artificial intelligence (AI) services for Internet of Everything (IoE) applications, such as sustainable cities, connected autonomous systems, brain-computer interfaces, digital twins, extended reality (XR), the metaverse, and e-health \cite{6Gwhitepaper, SaadNetwork2020}. However, existing cloud AI encounters challenges such as high latency, privacy leakage, and limited wireless resources. Consequently, edge AI, a paradigm that shifts AI capabilities from central cloud infrastructure to the network edge, has garnered significant interest by considering the integrated sensing, communication, and computation (ISCC) design \cite{ZhangWC2023, FengNetwork2021, ZhuSCIC2023}. By enabling continuous data acquisition, leveraging distributed computational resources, and employing efficient communication techniques at the edge, ISCC holds considerable potential for enhancing 6G network capabilities, optimizing resource utilization, and facilitating ubiquitous AI services \cite{LetaiefJSAC2022}.

Federated learning (FL) is recognized as a key technology for edge AI due to its privacy-preserving capabilities \cite{McMahanAISTATS2017}. FL enables collaborative model training across decentralized devices without requiring data to be centralized, thus maintaining data privacy and security. This decentralized approach is particularly attractive for intelligent applications in 6G networks, where data privacy and security are critical concerns. Currently, FL is primarily implemented using two algorithms: Federated Averaging (FedAVG) and Federated Stochastic Gradient Descent (FedSGD). FedAVG performs multiple local updates and transmits the updated local model parameters to a central server for global aggregation. In contrast, FedSGD computes the gradient of the local model and transmits these gradient parameters to the server, where they are aggregated before updating the global model. Existing research suggests that FedAVG is more communication-efficient than FedSGD, making it a key focus of ongoing research efforts.\cite{McMahanAISTATS2017, LiSPM2020, ChenTWC2021, WangJSAC2019, JiangTMC2024}.

While the ISCC design offers potential advantages for intelligent applications in 6G, integrating ISCC into existing FL approaches, leading to the development of Federated learning with ISCC (FL-ISCC), requires a reevaluation of their relative performance, which remains insufficiently explored.  
The question of whether FedAVG-ISCC and FedSGD-ISCC can achieve efficiency and cost-effectiveness within the ISCC framework is not well understood. This study addresses this gap by evaluating the performance of these two algorithms under the ISCC framework, focusing on training efficiency and performance analysis. The aim is to provide insights into FL-ISCC, thereby guiding the development of more efficient FL systems for edge AI in the 6G era.


\subsection{Related work}
FL has garnered significant attention for its ability to collaboratively train ML models across distributed devices without share raw data. However, communication overhead remains a major bottleneck due to the frequent exchange of models or gradients between the server and devices, particularly in large-scale scenarios \cite{LiSPM2020}. To mitigate this issue, considerable efforts have been directed towards the co-design of communication and computation \cite{ChenTWC2021, WangJSAC2019, JiangTMC2024}. One intuitive method to reduce communication overhead is to increase the number of local update steps in each communication round. For example, the authors in \cite{WangJSAC2019} analyzed the impact of the number of local steps on the convergence rate, leading to the development of an algorithm that adaptively determines the number of local updates by jointly considering communication and computation resource constraints. Another approach is to reduce the communication volume during each round. For instance, model pruning and quantization have been employed to enhance communication efficiency in FL \cite{JiangTMC2024, ShlezingerICASSP2020}. Specifically, Specifically, considering the heterogeneous capabilities of each device, the authors in \cite{JiangTMC2024} adopted a Multi-Armed Bandit-based online algorithm to determine the pruning ratios for each device after analyzing the impact of these ratios on training performance.  In \cite{ShlezingerICASSP2020}, a quantization scheme based on universal quantization theory was designed, yielding substantial performance gains compared to previous quantization approaches.

Nevertheless, the aforementioned works primarily employed a transmit-then-compute scheme for model or gradient transmission and aggregation over orthogonal resources, which results in average latency. To enhance communication efficiency, over-the-air federated learning (OTA-FL) has emerged as a promising solution \cite{YangTWC2020, ZhuTWC2020}. OTA-FL utilizes a transmit-while-compute scheme, enabling multiple devices to simultaneously transmit and aggregate their models or gradients using the same time-frequency resources through an integrated communication and computation design. In general, the implementation of OTA-FL can be categorized into over-the-air FedSGD (OTA-FedSGD) and over-the-air FedAVG (OTA-FedAVG). In OTA-FedSGD, gradients are computed after a single local step on the full data batch at each device and then aggregated over the air \cite{MohammadiTSP2020}. In OTA-FedAVG, local models are updated through multiple steps with mini-batches of data at each device. Instead of transmitting gradients, OTA-FedAVG uploads the updated models to the server for aggregation over the air \cite{LiTWC2024}.
Recent research has made significant efforts in facilitating OTA-FL in practice \cite{CaoIEEEWC2024}. For example, it has been demonstrated that OTA-FL can achieve privacy "for free" as long as the privacy constraint level is below a threshold that decreases with the signal-to-noise ratio \cite{LiuJSAC2021}. 

However, a critical challenge in OTA-FL is the aggregation error caused by channel noise perturbation, which can lead to transmission distortion and degrade FL performance. 
To mitigate gradient distortion in OTA-FedSGD, several power control approaches have been proposed \cite{ZhangTWC2021, CaoJSAC2022}. Specifically, the authors in \cite{CaoJSAC2022} revealed that training converges to the optimal point if the aggregation errors are unbiased; otherwise, it converges with an error floor. Based on this insight, power control algorithms have been developed for both biased and unbiased cases. In \cite{ZhangTWC2021}, the power control problem was investigated, and the optimal policy was derived in closed-form by taking gradient statistics into account.
To improve the convergence rate of OTA-FedAVG,  the work in \cite{CaoJSAC2022T} designed a joint power allocation and local step control duw to the fact that the convergence behavior is influenced by both the number of local updates and model distortion. Additionally, the authors in \cite{GafniTSP2024} proposed a Bayesian aggregation scheme that accounts for local steps and channel conditions.
However, few of these efforts consider the sensing process for data acquisition, resulting in a lack of investigation into FL-ISCC.

\subsection{Motivation and contribution}
Although the ISCC design has garnered increasing interest for enhancing network capabilities in the forthcoming 6G era \cite{QiTcomm2022, ZhaoTWC2022, HeTWC2024, WenTWC2023, LiTWC2023}, research on FL-ISCC remains in its early stages. Only a few studies have recently explored FL-ISCC \cite{LiuJSTSP2023, LiangICC2023}. Specifically, the work in \cite{LiangICC2023} theoretically analyzed the convergence of FedAVG-ISCC, where the effect of sample size collected in each communication round on convergence rate was explored. Furthermore, the work \cite{LiuJSTSP2023} proposed a joint resource allocation strategy for communication, computation, and sensing in a FedAVG-ISCC system constrained by training latency and energy.
Nevertheless, the theoretical understanding of the performance and advantages of FedAVG-ISCC and FedSGD-ISCC remains an open research area.

Motivated by these issues, we investigate a general FL-ISCC framework as outlined in \cite{LiuJSTSP2023} and \cite{LiangICC2023}. As shown in \ref{fig_sim}, the framework comprises an edge server and multiple devices, each equipped with sensing, communication, and computation capabilities. In each communication round, after updating the global model from the edge server, each device first performs sensing to collect samples from its surrounding environment. Subsequently, the device trains a local ML model using the collected samples and the available computational resources. Finally, efficient parameter aggregation is conducted with the aid of over-the-air computation through a wireless channel.
The main contributions of this work are summarized as follow:

\begin{itemize}

    \item \textbf{ Effectiveness of FL-ISCC:} We first implement both FedAVG and FedSGD algorithms within the proposed ISCC framework and validate its effectiveness in reducing model training latency and energy consumption, thereby significantly enhancing the training efficiency of FL.
    
	\item \textbf{Convergence analysis for FL-ISCC:} We then analyze the convergence performance of both FedAVG-ISCC and FedSGD-ISCC, considering the impact of sample collection size and communication errors in each communication round. Our analysis indicates that both the sample collection strategy and aggregation errors can deteriorate convergence performance. 
	
	
	\item \textbf{Performance Comparison:} Based on our convergence analysis, we further compare the performance of FedAVG-ISCC and FedSGD-ISCC. The results indicate that FedAVG-ISCC outperforms FedSGD-ISCC in IID settings due to the advantage of multiple local updates. However, in Non-IID settings, these local updates in FedAVG-ISCC amplify the effects of Non-IID data, leading to significant performance degradation. In contrast, FedSGD-ISCC, which does not involve multiple local updates, avoids this issue and is therefore more robust in Non-IID settings compared to FedAVG-ISCC. Additionally, the learning rate in FedAVG-ISCC can exacerbate communication errors, resulting in degraded performance. FedSGD-ISCC, which does not face these challenges, proves to be more robust to communication errors than FedAVG-ISCC.
	
	\item \textbf{Performance evaluation:} Finally, extensive simulations are conducted to evaluate the proposed FL-ISCC design. The results confirm the superiority of our framework in enhancing training efficiency and validate the theoretical analysis.
	
\end{itemize}

\begin{figure}[!t]
\centering
\includegraphics[width=0.9\linewidth]{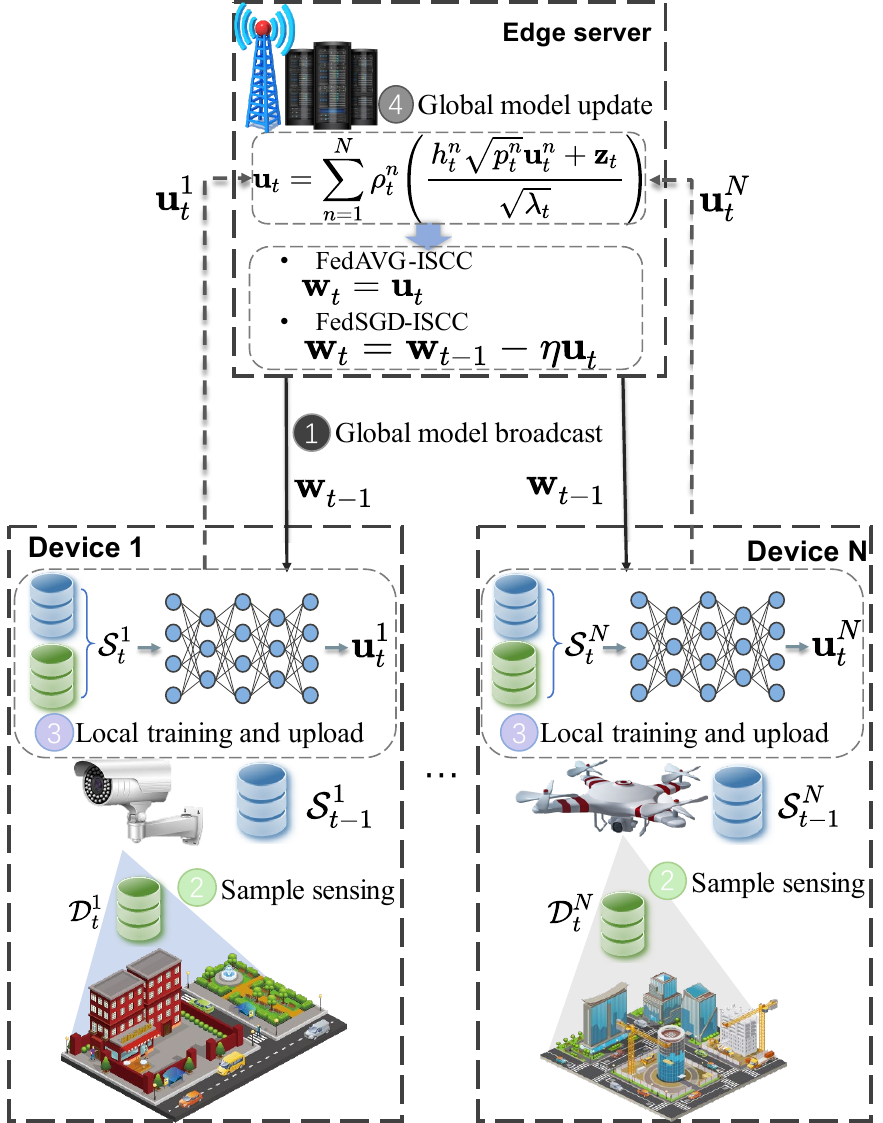}
\caption{Illustration of the  proposed SC$^2$-FEEL.}
\label{fig_sim}
\end{figure}

\section{System model}
In this section, we introduce the FL-ISCC system as shown in Fig. \ref{fig_sim}, where sensing, computation, and communication are
jointly considered. We first introduce the proposed FL-ISCC framework. Subsequently, we present the learning models of FedAvg and FedSGD, respectively. Thereafter, we present the communication model and computation model, respectively.

\subsection{FL-ISCC framework}
In this work, we consider an FL-ISCC framework comprising a single edge server and a set of N edge devices, denoted by $\mathcal{N} \triangleq \left\lbrace 1, 2, ..., N \right\rbrace $, collaborating to train a desired ML model.
We assume that both the edge server and the devices are equipped with a single antenna for parameter transmission. Each device is also capable of collecting samples from its surrounding environment, enabling the ML model to be trained on datasets continuously accumulated at the devices throughout the training process.

As illustrated in Fig. 1, the desired ML model, denoted by $\mathbf{w} \in \mathbb{R}^q$, where $q$ represents the model size, is trained over $T$ communication rounds. During any communication round $t\in \mathcal{T}\triangleq \{1, 2, ..., T\}$, the model parameters are updated through four key steps, detailed as follows:
\begin{itemize}
	\item[(1)] \textbf{Global model broadcast:} The edge server broadcasts the current global model parameters $\mathbf{w}^{}_{t-1}$ to all devices. Each device then updates its local model parameters $\mathbf{w}^{n}_{t}$ based on the received global parameters, i.e., $\mathbf{w}^{n}_{t} = \mathbf{w}^{}_{t-1}$.
	
	\item[(2)] \textbf{sample sensing:} Each device $n$ collects new samples $\mathcal{D}^{n}_{t}$ with the size of $D^{n}_{t} = | \mathcal{D}^{n}_{t} |$ from its surrounding environment. As a result, device $n$ maintains a cumulative dataset $\mathcal{S}^{n}_{t}$ that includes both the newly collected dataset $\mathcal{D}^{n}_{t}$ from the current round and the cumulative dataset $\mathcal{S}^{n}_{t-1} = \sum_{i=1}^{t-1}\mathcal{D}^{n}_{i}$ from the previous rounds, i.e., $\mathcal{S}^{n}_{t} = \mathcal{S}^{n}_{t-1}+\mathcal{D}^{n}_{t}$.
	
	\item[(3)] \textbf{Local training:} Based on the cumulative dataset $\mathcal{S}^{n}_{t}$, each device $n$ trains its local ML model. Let $F\left(\mathbf{w}^{n}_{t}; \mathcal{S}^{n}_{t}\right) $  denote the loss function for device $n$ over dataset $\mathcal{S}^{n}_{t}$, which is expressed as
    \begin{equation}\label{localLoss}
    	F\left(\mathbf{w}_t^n ; \mathcal{S}_t^n\right)=\frac{1}{S_t^n} \sum_{\left(\mathbf{x}_j, y_j\right) \in \mathcal{S}_t^n} f\left(\mathbf{w}_t^n,\left(\mathbf{x}_j, y_j\right)\right),
    \end{equation}
    where $\left(\mathbf{x}_{j}, y_j \right) $ is the $j$-th sample of dataset $\mathcal{S}_{t}^{n}$ with data $\mathbf{x}_{j}$ and label $y_j$, $f\left(\mathbf{w}_t^n ; \mathcal{S}_t^n\right)$ is the $j$-th sample-wise loss function, $S^{n}_{t} = | \mathcal{S}^{n}_{t} |$ is the size of dataset $\mathcal{S}^{n}_{t}$, and $S^{n}_{t} = S^{n}_{t-1}+ D^{n}_{t}$. Based on this loss function, both FedAVG and FedSGD perform local training to update the parameters $\mathbf{u}^{n}_{t}$. 

	\item[(4)] \textbf{Global model update:} After local training, all devices synchronously upload their updated local parameters $\mathbf{u}^{n}_{t}$ to the edge server for aggregation. The edge server then aggregates these parameters as $\mathbf{u}^{}_{t}$, given by
    \begin{equation}\label{AggregatedParameter}
        \mathbf{u}^{}_{t} = \rho^{n}_{t} \mathbf{u}^{n}_{t}.
    \end{equation}
    where $\rho^{n}_{t} = \frac{S^{n}_{t}}{S_{t}}$ with $S_{t} = \sum_{n=1}^{S^{n}_{t}}$ representing the total size of the accumulated dataset across all $N$ devices in the $t$-th round.
     Then, the global model $\mathbf{w}^{}_{t}$ is updated based on $\mathbf{u}^{n}_{t}$, resulting in the global loss function as
\begin{equation}\label{finalloss}
	F\left(\mathbf{w}_{t}^{ }; \mathcal{S}_{t}\right) =\sum_{n =1}^{N}\rho^{n}_{t} F\left(\mathbf{w}_t^n ; \mathcal{S}_t^n\right).
\end{equation}
The global model update for FedAVG and FedSGD will be detailed in the subsequent subsection.
\end{itemize}
This iterative process repeats over $T$ communication rounds, optimizing the model parameter $\mathbf{w}^{}_{t}$ to satisfy the following condition:
\begin{equation}
	\mathbf{w}^* \triangleq \arg \min _{\mathbf{w}} F\left(\mathbf{w}_T ; \mathcal{S}_T\right).
\end{equation}

\subsection{Global model update for FedAVG-ISCC}

In FedAVG, each device $n$ updates its local model $\mathbf{w}^{n}_{t}$ through multiple epochs using a stochastic gradient descent (SGD) method. Specifically, each device randomly selects mini-batches of samples $\xi^n$ from its dataset $\mathcal{S}_{t}^{n}$ for several local updates. The updated local model parameters $\mathbf{w}^{n}_{t}$ are then uploaded to the server for aggregation. Thus, the parameters $\mathbf{u}^{n}_{t}$ to be uploaded are the updated local model $\mathbf{w}^{n}_{t}$, which is expressed as
\begin{equation}
	\begin{aligned}
		\mathbf{u}^{n}_{t} = \mathbf{w}_{t}^n & =\mathbf{w}_{t-1}-\eta \sum_{i=1}^{\tau_{t}} \nabla F\left(\mathbf{w}_{t-1,i}^n ; \xi_{}^{n}\right), i=1,2, \ldots \tau_{}, \\
	\end{aligned}
\end{equation}
where $\eta$ is the learning rate, $\tau_{}$ is the number of local updates. 
As a result, the global model is updated based on \eqref{AggregatedParameter} as
\begin{equation}\label{fedavg_aggre}
	\mathbf{w}^{}_{t} = \mathbf{u}^{}_{t} = \sum_{n =1}^{N}\rho^{n}_{t}\mathbf{w}^{n}_{t},
\end{equation}

\subsection{Global model update for FedSGD-ISCC}

Unlike FedAVG, which performs multiple updates on the local model, FedSGD calculates and uploads the gradients of the local model, denoted by $\nabla F\left(\mathbf{w}_{t}^n ; \mathcal{S}^{n}_{t}\right)$, using the full batch of dataset $\mathcal{S}^{n}_{t}$. Therefore, the parameters to be uploaded are given by $\mathbf{u}^{n}_{t} = \nabla F\left(\mathbf{w}_{t}^n ; \mathcal{S}^{n}_{t}\right)$.

All devices then transmit their local gradient parameters to the edge server for aggregation. Consequently, \eqref{AggregatedParameter} is reformulated as:
\begin{equation}\label{globalGrad}
	\mathbf{u}^{}_{t} = \sum_{n=1}^{N} \rho^{n}_{} \nabla F\left(\mathbf{w}_{t-1}^{n} ; \mathcal{S}^{n}_{t}\right). 
\end{equation}

Based on \eqref{globalGrad}, the global model $\mathbf{w}_{t}$ at the edge server is updated as
\begin{equation}\label{Update_fedsgd}
    \begin{split}
	\mathbf{w}_{t} & = \mathbf{w}_{t-1} - \eta \mathbf{u}^{}_{t}  = \mathbf{w}_{t-1} - \eta \sum_{n=1}^{N} \rho^{n}_{} \nabla F\left(\mathbf{w}_{t-1}^{n} ; \mathcal{S}^{n}_{t}\right).
    \end{split}
\end{equation}

\subsection{Communication Model}

In this work, we employ over-the-air computation for communication-efficient parameter transmission in FL. 
Assume that $\hat{h}^{n}_{t}$ is the complex channel coefficient from device $n$ to the edge server during $t$-th communication round. The magnitude of this coefficient, $h^{n}_{t} = |\hat{h}^{n}_{t}|$, can be estimated by each device. Let the parameters transmitted from each device to the server be denoted by $\mathbf{u}^{n}_{t}$.
The received signal (after phase compensation) at the edge server is given by
\begin{equation}
	\mathbf{y}^{\text{comm}}_{t} = \sum_{n=1}^{N} \left(h^{n}_{t} \sqrt{p^{n}_{t}}\mathbf{u}^{n}_{t} + \mathbf{z}_{t} \right) 
\end{equation}
where $p^{n}_{t}$ represents the transmit power scaling factor of device $n$ at $t$th communication round, and $\mathbf{z}_{t} \in \mathbb{R}^q$ denotes the additive white Gaussian noise, following a Gaussian distribution with zero mean and $\sigma_z$, i.e., $\mathbf{z}_{t} \sim \mathcal{CN}\left( 0, \sigma_z\mathbf{I}\right) $. 
For FedAVG, the transmitted parameters are the local model parameters, thus $\mathbf{u}^{n}_{t} = \mathbf{w}_{t}^{n}$. 
For FedSGD, the transmitted parameters are the local gradients, thus $\mathbf{u}^{n}_{t} = \nabla F\left(\mathbf{w}_{t}^n ; \mathcal{S}^{n}_{t}\right)$.

Note that the over-the-air computation transmits the signals in an analog manner. Specifically, each element of the parameters is modulated as a single analog symbol for transmission. Therefore, to achieve the parameter transmission, the total number of analog symbols to be transmitted is $q$ for a parameter size of $q$. Let $L$ represent the number of symbols in each resource block with duration $T_{slot}$. As a result, for $t$-th communication round, the transmission latency is thus expressed as
\begin{equation*}
	t^{\text{comm}}_t = ceil\left( \frac{q}{L} \right) T_{slot}
\end{equation*}
where $ceil(.) $ is the integer ceiling function, $T_{slot} = 1$ ms and $L$ = 14.
Therefore, the communication energy consumption of device $n$ is given by $e^{\text{comm}}_{n,k} = p^{k}_{n} t^{\text{comm}}_k$.

To recover the parameters of interest from the wireless channels, we apply a denoising factor  $\lambda_t$ to the received signals, 
\begin{equation}
	\mathbf{u}_{t} = \sum_{n=1}^{N} \frac{\rho^{n}_{t} \mathbf{y}^{\text{comm}}_t}{\sqrt{\lambda_t}}=  \sum_{n=1}^{N} \rho^{n}_{t} \left( \frac{ h^{n}_{t} \sqrt{p^{n}_{t}}\mathbf{u}^{n}_{t} + \mathbf{z}_{t}}{\sqrt{\lambda_t}} \right). 
\end{equation}
The edge server aims to obtain the ML model parameter vector unaffected by noise, i.e.,
\begin{equation}
	\tilde{\mathbf{u}}_t = \sum_{n=1}^{N}\rho^{n}_{t} \mathbf{u}^{n}_{t}.
\end{equation}

As a result, the communication error $\varepsilon_t$ during $t$-th communication round is defined to quantify the model parameter vector distortion caused by channel noise, which is shown as
\begin{equation}\label{commerr}
	\begin{split}
		\boldsymbol{\varepsilon}_t &= \mathbf{u}_{t} - \tilde{\mathbf{u}}_{t}\\
		& = \sum_{n=1}^{N}\rho^{n}_{t}\left( \frac{h^{n}_{t}\sqrt{p^{n}_{t}}}{\sqrt{\lambda_t}} -1 \right)\mathbf{u}^{n}_{t} + \frac{1}{\sqrt{\lambda_t}} \mathbf{z}_{t}.
	\end{split}
\end{equation}

\subsection{Computation Model}

In $t$-th communication round, device $n$ performs local training with dataset $\mathcal{S}^{n}_{t}$. Let $\xi^n$ denote the number of CPU cycles for device $n$ to execute one sample. $f^{n}_{t}$ denote the CPU-cycle frequency of the device $n$. $\varsigma^n$ is the energy consumption coefficient depending on the chip of each device. Then, the computation latency of device $n$ for one local training epoch is given by
\begin{equation*}
	t^{n,\text{comp}}_{t} = \frac{\xi_n\sum_{i=1}^{t}D^n_i}{f^{n}_{t}} = \frac{\xi^n S^{n}_{t}}{f^{n}_{t}}.
\end{equation*}

Meanwhile, the computation energy consumption of device $n$ for one local training epoch can be expressed as: 
\begin{equation*}
	e^{n,\text{comp}}_{t}=\xi^n \varsigma^n (f^{n}_{t})^{2} \sum_{k=1}^{k}D^{n}_{i} = \xi_n \varsigma^n (f^{n}_{t})^{2} S^{n}_{t}.
\end{equation*}

Based on the communication model and computation model, we can respectively derive the training latency and energy consumption as
\begin{equation*}
	t_{t} = \max_{n \in \mathcal{N}} \{t^{n,\text{comp}}_{t} \} + t^{\text{comm}}_t,
\end{equation*}
and 
\begin{equation*}
	e_{t} = \sum_{n=1}^{N}\left(e^{n,\text{comp}}_{t} + e^{n,\text{comm}}_{t} \right).
\end{equation*}
\section{Convergence Analysis and Performance Evaluation}
In this section, we provide the convergence analysis of the proposed FL-ISCC, i.e., both FedAVG-ISCC and FedSGD-ISCC. Based on the convergence results, we compare the performance between FedAVG-ISCC and FedSGD-ISCC from a theoretical perspective. 

\subsection{Assumptions and preliminaries} %
To facilitate the convergence analysis, we first introduce the following assumptions for the loss functions, which are commonly adopted in existing works, such as \cite{WangJSAC2019, CaoJSAC2022, WangJMLR2021, YuICML2019, WangNeuIPS2020}.  

\textbf{Assumption 1 (L-smoothness).} 
The loss function, $F(\mathbf{w}_t; \mathcal{S}_t), \forall t$, is either continuously differentiable or Lipschitz continuous with a non-negative Lipschitz constant $L \geq 0$, which can be formulated as  
\begin{equation}\label{Lsmooth}
\begin{split}
	F(\mathbf{w}_t; \mathcal{S}_t) &\leq F(\mathbf{v}_t; \mathcal{S}_t)+ \left\langle \nabla F (\mathbf{v}_t; \mathcal{S}_t), (\mathbf{w}_t-\mathbf{v}_t) \right\rangle \\
  & ~~~ + \frac{L}{2}||\mathbf{w}_t-\mathbf{v}_t||^2, \forall \mathbf{w}_t, \mathbf{v}_t \in \mathbb{R}^q, t,
\end{split}
\end{equation}
where $\nabla F(\mathbf{v}_t; \mathcal{S}_t)$ denotes the gradient of $F(\mathbf{v}_t; \mathcal{S}_t)$.

\textbf{Assumption 2 (Gradient bound).} 
For any dataset $\mathcal{S}_t$ at $t$-th communication round, the expected squared norm of gradient $\nabla F(\mathbf{w}_t; \mathcal{S}_t)$ is bounded by a positive constant $G_t$, namely,
\begin{equation}
	\mathbb{E}\left( \left\| \nabla F(\mathbf{w}_{t};\mathcal{S}_{t}) \right\|^2\right) \leq G_t.
\end{equation}

\textbf{Assumption 3 (Unbiased gradient and bounded variance).} 
For each device $n$, the stochastic gradient is unbiased, i.e., $\mathbb{E} \left(\nabla F\left(\mathbf{w}_{t-1}^n ; \xi_{t}^{n}\right) \right) = \nabla F\left(\mathbf{w}_{t-1}^n ; \mathcal{S}_{t}^{n}\right) $. Moreover, 
the variance of stochastic gradients of each client is bounded by
\begin{equation}
	\mathbb{E} \left(  \nabla F\left(\mathbf{w}_{t-1}^n ; \xi_{t}^{n}\right) - \nabla F\left(\mathbf{w}_{t-1}^n ; \mathcal{S}_{t}^{n}\right) \right)  \leq \sigma_{}^{2}, 
\end{equation}
where $\sigma_{n}^{2}$ is a non-negative constant.

\textbf{Assumption 4 (Bounded dissimilarity).} 
For any $F(\mathbf{w}_t; \mathcal{S}_t)$ and $\nabla F\left(\mathbf{w}_{t-1}^{n} ; \mathcal{S}_{t}^{n}\right)$, there exist constants $\alpha^{2} \geq 1$ and $\beta^{2} \geq 0$ to quantify the degree of non independent and identically distributed (Non-IID) datasets such that
\begin{equation}
    \sum_{n = 1}^{N} \rho^{n} \left\| \nabla F\left(\mathbf{w}_{t-1}^{n} ; \mathcal{S}_{t}^{n}\right) \right\|^2 \leq \alpha^{2} \left\| F(\mathbf{w}_t; \mathcal{S}_t) \right\|^2  + \beta^{2}.
\end{equation}
If the datasets among devices are IID setting, then we have $\alpha^{2} = 1$ and $\beta^{2} = 0$.


In both FedAVG-ISCC and FedSGD-ISCC, the model parameter $\mathbf{w}^{n}_{t}$ is updated based on the cumulative dataset $\mathcal{S}_{t-1}^n$ and the newly sensed dataset $\mathcal{D}^{n}_{t}$ in each round. Therefore, it is essential to discuss the impact of these datasets on the improvement of the global loss function. To this end, we first examine how these datasets affect the gradients used for model updating. This leads us to introduce Lemma \ref{lemma1}.

\begin{lemma}\label{lemma1}
	Given the datasets $\mathcal{S}^{n}_{t-1}$ and $\mathcal{D}^{n}_{t}$ in the $t$-th communication round, the aggregated gradient $\sum_{n=1}^{N} \rho^{n} \nabla F(\mathbf{w}^{n}_{t-1}; \mathcal{S}^{n}_{t})$ satisfies the following equation.
	\begin{equation}
            \begin{split}
              &\sum_{n=1}^{N} \rho^{n} \nabla F(\mathbf{w}^{n}_{t-1}; \mathcal{S}^{n}_{t}) = \\ 
              &\frac{S_{t-1}}{S_{t}} \sum_{n=1}^{N} \Bar{\rho}^{n} \nabla F(\mathbf{w}^{n}_{t-1}; \mathcal{S}^{n}_{t-1}) 
              + \frac{D_t}{S_{t}} \sum_{n=1}^{N} \Tilde{\rho}^{n} \nabla F(\mathbf{w}^{n}_{t-1}; \mathcal{D}^{n}_t),
            \end{split}
	\end{equation}
	where $D_t = \sum^{N}_{n=1} D^{n}_{t}$, $\Bar{\rho}^{n} = \frac{{S}^{n}_{t-1}}{{S}^{}_{t-1}}$ and $\Tilde{\rho}^{n} = \frac{{D}^{n}_{t}}{{D}^{}_{t}}$.
\end{lemma}
\begin{proof}[proof]
%
%
	
		Please refer to Appendix \ref{AppenA}.
\end{proof}

\subsection{Convergence analysis for FedAVG-ISCC}
In this subsection, we analyze the convergence rate of FedAVG-ISCC. We first introduce Lemma \eqref{lemma2} to show the upper bound of the improvement of the global loss function based on Lemma \ref{lemma1}.
\begin{lemma}\label{lemma2}
When the learning rate $\eta$ satisfies $0 \leq  2 L^2 \eta^2 \tau_{}\left( \tau_{} -1\right) \leq \min \{\frac{1}{5} ,\frac{S^{2}_{t}}{S^{2}_{t} + 4 S^{2}_{t-1}} \} $ in the $t$-th communication round, the improvement of the global loss function is bounded by \eqref{LossImprove1}.
	\begin{figure*}[htbp] 
		\centering
\begin{equation}\label{LossImprove1}
    \begin{aligned}
         & F\left(\mathbf{w}_{t}; \mathcal{S}_{t}\right) - F\left(\mathbf{w}_{t-1};\mathcal{S}_{t-1}\right)  \le
         \begin{cases}
             \begin{matrix} 
                       -  \frac{\eta\tau_{}}{4} \left( 2 -  \alpha^2  \right) \left\| \nabla F(\mathbf{w}_{0};\mathcal{D}_{1}) \right\|^2  + L \tau^{}_{} \eta^2 \sigma^2 \sum_{n=1}^{N} \left( \rho^{n}\right)^2 +  \frac{ 1 }{ \eta^{}_{} \tau^{}_{}} \left\| \boldsymbol{\varepsilon}_t \right\|^2 \\
                       + \frac{ 5L^2 \eta^{3}_{} \sigma^2 \tau^{}_{}  \left( \tau_{} -1\right)}{4}+ \frac{ \beta^2  \eta^{}_{}\tau^{}_{} }{4}
             \end{matrix}, 
             & \mathrm{if} ~~t=1,\\
             \begin{matrix}
                    -\frac{\tau^{}_{} \eta}{ 4 } \left( 2 -      \alpha^2    \right)  \left\| \nabla F(\mathbf{w}_{t-1};\mathcal{S}_{t-1}) \right\|^2  + L \tau^{}_{} \eta^{2}_{} \sigma^2 \sum_{n=1}^{N} \left( \rho^{n}\right)^2 +  \frac{ 1 }{ \eta^{}_{} \tau^{}_{}} \left\| \boldsymbol{\varepsilon}_t \right\|^2  \\
                    +  \left( 1+ \frac{ S^{2}_{t} }{ 4 S^{2}_{t-1} } \right)  L^2 \eta^3 \sigma^2 \tau_{} \left( \tau_{} -1\right) + \frac{ \eta^{}_{}\tau^{}_{}S^{2}_{t}  }{4S^{2}_{t-1}}  \beta^2   +  \frac{ \eta^{}_{}\tau^{}_{} \alpha^2 }{4}  \frac{D^{2}_{t}}{S^{2}_{t-1}} G_{t}
             \end{matrix},
             &\mathrm{otherwise}.
         \end{cases}
 \end{aligned}
\end{equation}
	\end{figure*}

\end{lemma}
\begin{proof}[proof]
Please refer to Appendix \ref{AppenB}.
\end{proof}

	The average-squared gradient norm is widely adopted to depict the performance of FL \cite{ LiuJSTSP2023}. Based on Lemma \ref{lemma1} and Lemma \ref{lemma2}, we introduce the following Theorem to show the upper bound of the average squared gradient norm, which illustrates the convergence performance for FedAVG-ISCC.
\begin{theorem} \label{Theo1}
	Under the condition of $0 \leq  2 L^2 \eta^2 \tau_{}\left( \tau_{} -1\right) \leq \min \{\frac{1}{5} ,\frac{S^{2}_{t}}{S^{2}_{t} + 4 S^{2}_{t-1}} \}$, the average squared gradient norm after $T$ communication rounds is bounded by
    \begin{equation}\label{gradnorm}
        \begin{split}
            \frac{1}{T}\sum_{t=1}^{T} \mathbb{E}\left\|\nabla F(\mathbf{w}_{t-1};\mathcal{S}_{t-1})\right\|^2 \leq \underbrace{\frac{4 \left( F\left(\mathbf{w}_{0};\mathcal{S}_{0}\right) - F^* \right) }{\left( 2 - \alpha^2  \right) T\eta\tau_{} }  }_{\text{\rm Effects of initialization}} \\
            + \underbrace{ \frac{  4 \sum_{t=1}^{T} \mathbb{E} \left\| \boldsymbol{\varepsilon}_t \right\|^2 }{\left( 2 - \alpha^2  \right) T\eta^{2} \tau^{2}  } }_{\text{\rm Effects of communication errors}}  + \underbrace{ \frac{4 L \eta \sigma^2 }{\left( 2 - \alpha^2  \right)} \sum_{n=1}^{N} \left( \rho^{n}\right)^2  }_{\text{\rm Effects of gradient variance}} \\
            +\underbrace{ \frac{1}{\left( 2 - \alpha^2  \right)  T} \left[ \left(1 + \sum_{t=2}^{T} \frac{ S^{2}_{t} }{ S^{2}_{t-1} } \right)  \beta^2 + \alpha^2 \sum_{t=2}^{T}  \frac{D^{2}_{t}}{S^{2}_{t-1}}  G_{t} \right]    }_{ \text{\rm Effects of sample sensing strategy and  NonIID  }} \\
            + \underbrace{ \frac{  L^{2}_{} \eta^{2}_{} \sigma^2 \left( \tau_{} -1\right) }{\left( 2 - \alpha^2  \right)  T} \left[  5 + \sum_{t=2}^{T} \left( 4 + \frac{S^{2}_{t}}{S^{2}_{t-1}} \right)  \right] }_{\text{\rm Effects of sample sensing strategy and local updates }}.
        \end{split}
    \end{equation}
\end{theorem}





\subsection{Convergence Analysis for FedSGD-ISCC}

Based on Lemma \ref{lemma1}, we introduce Lemma \ref{lemma3} to show the upper bound of the improvement of the global loss function. 
\begin{lemma}\label{lemma3}
	When the learning rate $\eta$ satisfies $0 \leq \eta \leq \min \{ \frac{1}{L}, \frac{1}{2\sqrt{2}L} \frac{S_{t}}{S_{t-1}}\} $ in the $t$-th communication round, the improvement of the global loss function is bounded by  
		\begin{equation}\label{LossImprove2}
			\begin{aligned}
		 		& F\left(\mathbf{w}_{t}; \mathcal{S}_{t}\right) - F\left(\mathbf{w}_{t-1};\mathcal{S}_{t-1}\right)  \le\\
		 		&\begin{cases}
		 			\begin{matrix} 
		 				- \frac{\eta }{4} \left( 2 - \alpha^2  \right)  \left\| \nabla F(\mathbf{w}_{0};\mathcal{D}_{1}) \right\|^2  + \eta \left\|  \boldsymbol{\varepsilon}_t  \right\|^2 +  \frac{ \eta \beta^2 }{4}
		 			\end{matrix}, & \mathrm{if} ~~t=1,\\
		 			\begin{matrix}
		 				 - \frac{\eta }{4} \left( 2 - \alpha^2  \right)  \left\| \nabla F(\mathbf{w}_{t-1};\mathcal{S}_{t-1}) \right\|^2  + \eta \left\|  \boldsymbol{\varepsilon}_t  \right\|^2  \\
        + \frac{ \eta}{4} \frac{S^{2}_{t}}{S^{2}_{t-1}} \beta^2 + \frac{\eta \alpha^2  }{4 } \frac{D^{2}_{t}}{S^{2}_{t-1}}  G_{t}
		 			\end{matrix},
		 			&\mathrm{otherwise}.
		 		\end{cases}
		 	\end{aligned}
		\end{equation}
\end{lemma}
\begin{proof}[proof]
	Please refer to Appendix \ref{AppenC}.
\end{proof}
	
	
Based on Lemma \ref{lemma1} and Lemma \ref{lemma3}, we introduce the following Theorem to show the upper bound of the average-squared gradient norm  for FedSGD-ISCC. 
\begin{theorem} \label{Theo2}
	Under the condition $0 \leq \eta \leq \min \{ \frac{1}{L}, \frac{1}{2\sqrt{2}L} \frac{S_{t}}{S_{t-1}}\}, \forall t $, the average-squared gradient norm after $T$ communication rounds is bounded by
		\begin{equation}\label{gradnorm2}
			\begin{split}
				&\frac{1}{T}\sum_{t=1}^{T} \mathbb{E}\left\| \nabla F(\mathbf{w}_{t-1};\mathcal{S}_{t-1})\right\|^2 \leq  \\
                    &\underbrace{ \frac{4 \left( F\left(\mathbf{w}_{0};\mathcal{S}_{0}\right) - F^* \right) }{\left( 2 - \alpha^2  \right) T\eta } }_{\text{\rm Effects of  Initialization}}  + \underbrace{\frac{4}{\left( 2 - \alpha^2  \right) T} \left[ \sum_{t=1}^{T} \mathbb{E}  \left(\left\|\boldsymbol{\varepsilon}_t\right\|^2\right)  \right]}_{\text{\rm Effects of communication errors}} \\ 
				&+ \underbrace{\frac{ 1 }{\left( 2 - \alpha^2  \right) T}\left[ \left( 1 + \sum_{t=2}^{T} \frac{S^{2}_{t}}{S^{2}_{t-1}} \right) \beta^2 + \alpha^2 \sum_{t=2}^{T} \frac{D^{2}_{t}}{S^{2}_{t-1}}  G_t  \right]  }_{\text{\rm Effects of sample sensing strategy  and  NonIID } }.
			\end{split}
	\end{equation}
\end{theorem}


\subsection{Performance analysis and discussion}

In this subsection, we analyze the performance of both FedAVG-ISCC and FedSGD-ISCC from a theoretical perspective. Specifically, we examine their performance considering differnernt factors including IID setting, Non-IID setting, communication errors and sample collection strategy, followed by a discussion of their computational complexity.

\subsubsection{Performance analysis to IID setting}
To analyze performance under IID data, we assume that both FedAVG-ISCC and FedSGD-ISCC employ the same sample collection strategy and experience identical communication errors during each communication round. For IID data, we have $\alpha^2 = 1$ and $\beta^2 = 0$. When the local updates satisfy $\tau \geq 1$, it is evident from \eqref{gradnorm} that the bound of FedAVG-ISCC first decreases and then increases with the increase in $\tau$. In contrast, the bound for FedSGD-ISCC in \eqref{gradnorm2} remains constant since it does not involve local updates ($\tau$). This indicates that FedAVG-ISCC can outperform FedSGD-ISCC due to the advantage of multiple local updates ($\tau$). Thus, we conclude that FedAVG-ISCC is more efficient and effective under IID data, when $\tau$ is appropriately tuned.

\subsubsection{Performance analysis to Non-IID setting}
For the analysis under Non-IID setting, we similarly assume that both FedAVG-ISCC and FedSGD-ISCC employ the same sample collection strategy and experience identical communication errors during each communication round, with a constant $\tau$.
Under Non-IID data, we have $\alpha^2 \geq 1$ and $\beta^2 \geq 0$. It is evident from \eqref{gradnorm} and \eqref{gradnorm2} that Non-IID data significantly deteriorates the performance of both FedAVG-ISCC and FedSGD-ISCC as $\alpha^2$ ($\alpha^2 \leq 2$) and $\beta^2$ increase. However, the last term in \eqref{gradnorm} shows that the local updates ($\tau$) in FedAVG-ISCC amplify the impact of Non-IID data, further exacerbating performance degradation. In contrast, FedSGD-ISCC, as seen in \eqref{gradnorm2}, avoids this issue by not involving multiple local updates. Consequently, while FedAVG-ISCC is more efficient under IID data, it is more vulnerable to Non-IID settings. Meanwhile, FedSGD-ISCC exhibits greater robustness under Non-IID data, as it does not suffer from the performance degradation caused by local updates.

\subsubsection{Performance analysis of communication errors}

We examine the performance of both FedAVG-ISCC and FedSGD-ISCC in the presence of communication errors, assuming identical sample collection strategies and data distribution.
From the second term in \eqref{gradnorm}, it is observed that communication errors in FedAVG-ISCC are divided by $T\eta^2$, whereas in FedSGD-ISCC, communication errors are averaged over $T$ rounds, as shown in \eqref{gradnorm2}. In practice, since the learning rate typically satisfies $0 \leq \eta \leq 0.1$, this significantly amplifies the communication errors in \eqref{gradnorm}, severely deteriorating the performance of FedAVG-ISCC. Conversely, FedSGD-ISCC does not encounter such issues. Therefore, we conclude that FedSGD-ISCC is more robust to communication errors than FedAVG-ISCC.

\subsubsection{Performance analysis of sample collection strategy }
It is clear from \eqref{gradnorm} and \eqref{gradnorm2} that the sample collection strategy plays a critical role in the performance of both FedAVG-ISCC and FedSGD-ISCC. Collecting more samples in the early communication rounds leads to better performance. However, this also increases computational overhead in terms of training latency and energy consumption. Therefore, a well-designed sample collection strategy is essential for achieving cost-efficient federated learning.

\subsubsection{Computational Complexity}
We next analyze the computational complexities of FedAVG-ISCC and FedSGD-ISCC, which are formally stated in the following corollaries. 

\begin{corollary}
	Suppose the learning rate satisfies $\eta = \sqrt{\frac{N}{\tau T}}$ and let $\rho^{n} = \frac{1}{N}$. Then, for sufficiently large $T$, the computational complexity of FedAVG-ISCC is given by
	\begin{equation}\label{CompAVG}
		\begin{split}
			\frac{1}{T}\sum_{t=1}^{T} \mathbb{E}\left\|\nabla F(\mathbf{w}_{t-1};\mathcal{S}_{t-1})\right\|^2 \leq \mathcal{O}\left( + \frac{  2 M_1 }{\left( 2 - \alpha^2  \right) N \tau^{}  }  + \frac{M_2}{\left( 2 - \alpha^2  \right) }  \right. \\
			\left. \frac{4 \left( F\left(\mathbf{w}_{0};\mathcal{S}_{0}\right) - F^* \right) }{\left( 2 - \alpha^2  \right) \sqrt{N\tau T}}  + \frac{L \sigma^2 }{\left( 2 - \alpha^2  \right) \sqrt{N\tau T}}  + \frac{  L^{2}_{} N \sigma^2 \left( \tau_{} -1 \right) M_3 }{\left( 2 - \alpha^2  \right) \tau T  }   \right)\\
                = \mathcal{O} \left( \frac{1 }{\sqrt{\tau N T}}  \right) + \mathcal{O} \left( \frac{ M_1 }{ N \tau }  \right) + \mathcal{O} \left( M_2 \right) + \mathcal{O} \left(  \frac{\sigma^2 N  \left( \tau_{} -1\right) M_3 }{\tau_{} T } \right)
		\end{split}
	\end{equation}
	where $\mathcal{O}$ swallows all constants (including $L$). $M_1 = \sum_{t=1}^{T} \mathbb{E} \left\| \boldsymbol{\varepsilon}_t \right\|^2$,  $M_3 = \frac{1 }{T} \sum_{t=1}^{T} \left[  5 + \sum_{t=2}^{T} \left( 4 + \frac{S^{2}_{t}}{S^{2}_{t-1}} \right)  \right]$, and $M_2 = \frac{1 }{T} \sum_{t=1}^{T} \left[ \left(1 + \sum_{t=2}^{T} \frac{ S^{2}_{t} }{ S^{2}_{t-1} } \right)  \beta^2 + \alpha^2 \sum_{t=2}^{T}  \frac{D^{2}_{t}}{S^{2}_{t-1}}  G_{t} \right]$. %
\end{corollary}


\begin{corollary}
	Suppose the learning rate satisfies $\eta = \sqrt{\frac{N}{ T}}$. Then, for sufficiently large $T$, the computational complexity of FedSGD-ISCC is given by
	\begin{equation}\label{CompSGD}
		\begin{split}
			&\frac{1}{T}\sum_{t=1}^{T} \mathbb{E}\left\|\nabla F(\mathbf{w}_{t-1};\mathcal{S}_{t-1})\right\|^2 \leq \\
                &\mathcal{O} \left( \frac{4 \left( F\left(\mathbf{w}_{0};\mathcal{S}_{0}\right) - F^* \right) }{\left( 2 - \alpha^2  \right) \sqrt{NT}} +  \frac{4 M_1}{\left( 2 - \alpha^2  \right) }  + \frac{M_2}{\left( 2 - \alpha^2  \right) } \right) \\
			  & = \mathcal{O} \left( \frac{1}{\sqrt{NT}} \right) + \mathcal{O} \left( M_1 \right) + \mathcal{O} \left( M_2 \right) .
		\end{split}
	\end{equation}
\end{corollary}

It is observed from \eqref{CompAVG} and \eqref{CompSGD} that the computational complexities of FedAVG-ISCC and FedSGD-ISCC largely depends on the communication errors ($M_1$) and sample sensing strategy ($M_2$ and $M_3$). If the communication errors are properly eliminated and the sample sensing strategy is bounded, the computational complexities of FedAVG-ISCC and FedSGD-ISCC are respectively represented as  $\mathcal{O} \left( \frac{1 }{\sqrt{\tau N T}}  \right) + \mathcal{O} \left(  \frac{\sigma^2 N  \left( \tau_{} -1\right) }{\tau_{} T } \right)$ and $\mathcal{O} \left( \frac{1}{\sqrt{NT}} \right)$, which are consistent with previous results \cite{WangJMLR2021, YuICML2019, WangNeuIPS2020}. Therefore, a well-designed communication optimization algorithm and sample sensing strategy are necessary to improve their computational complexities.

\section{Simulation results}
This section presents simulation results to demonstrate the efficiency of and validate the performance analysis of the ISCC design for FL.

\begin{figure}[]
 \begin{center}
  \subfigure[MNIST dataset]{
    \includegraphics[width=0.7\linewidth]{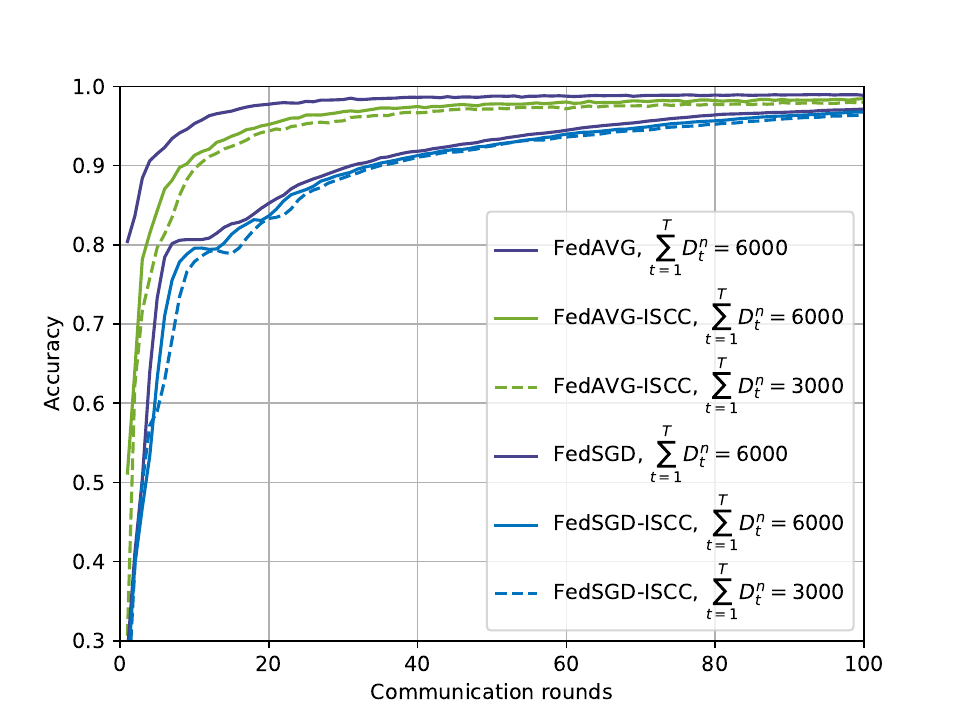}
  }%
  
  \subfigure[FMNIST dataset]{
    \includegraphics[width=0.7\linewidth]{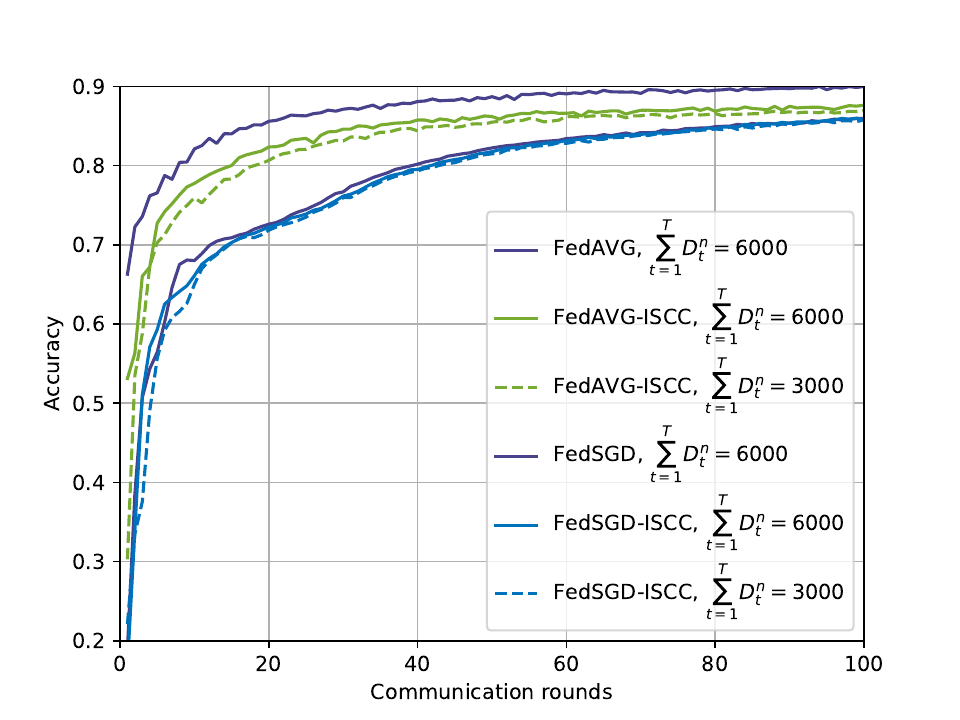}
  }%

  \caption{Performance comparison between FedSGD-ISCC and FedAVG-ISCC under IID settings.}
  \label{IID_simulation}
 \end{center}
\end{figure}

\subsection{Experiment setup}
In this work, we set the number of devices to $N=10$. The wireless channels between devices and the server follow independent and identically distributed (IID) Rayleigh fading, modeled as IID circularly symmetric complex Gaussian random variables with zero mean and unit variance. We assume the noise variance $\sigma^2_z = 1$ W, and the maximum transmit power budget of each device $P^{n}_{\max} = 10$ W unless otherwise specified. 
We consider to train a convolutional neural network (CNN) model for an image classification task. The performance analysis derived in this work is evaluated and validated over the MNIST and Fashion MNIST datasets, each containing 60,000 samples for training and 10,000 samples for testing. The learning rate is set to 0.001. we assume that the classification task requires collecting a total of ${ \textstyle \sum_{t=1}^{T}} D^n_t $ = 6000 samples for each device during the training procedure.


\subsection{Validation of theoretical analysis}

In Fig. \ref{IID_simulation}, we compare the accuracy of FedSGD-ISCC and FedAVG-ISCC across different datasets under IID settings, including MNIST and FMNIST. The results clearly show that FedAVG-ISCC significantly outperforms FedSGD-ISCC in IID settings, particularly in terms of faster convergence. This finding supports our theoretical analysis, which suggests that FedAVG-ISCC benefits from multiple local updates, leading to improved convergence performance. Additionally, we observe that both FedAVG-ISCC and FedSGD-ISCC exhibit worse convergence compared to the classic FedAVG and FedSGD, respectively. This is consistent with our analysis, which indicates that communication errors and the sample sensing strategy negatively impact the performance of FL.

\begin{figure}[]
 \begin{center}
  \subfigure[MNIST dataset]{
    \includegraphics[width=0.7\linewidth]{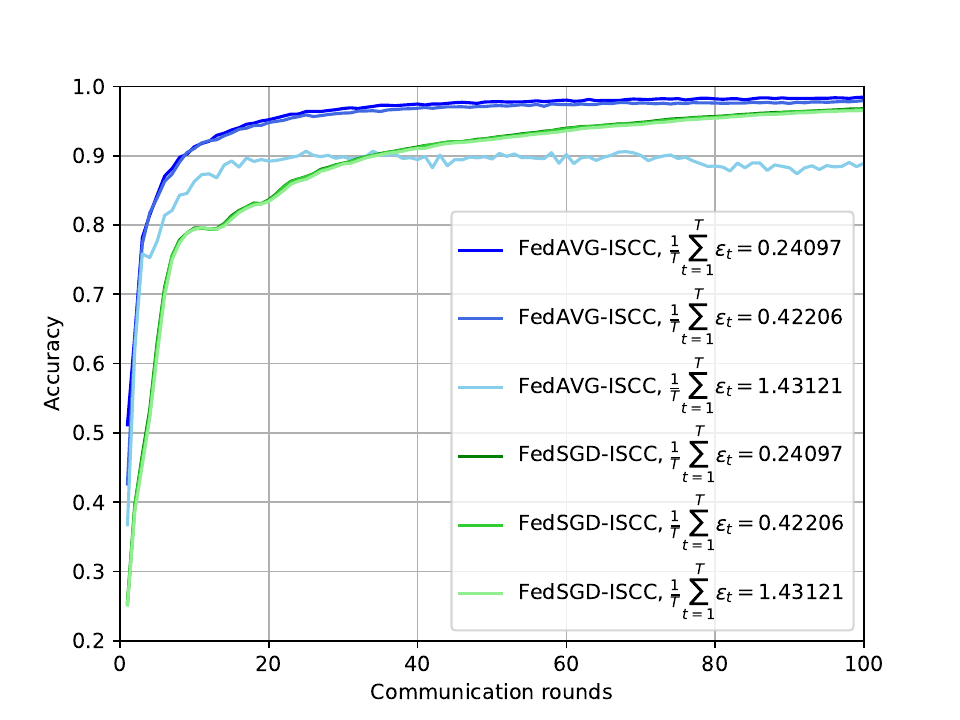}
  }%
  
  \subfigure[FMNIST dataset]{
    \includegraphics[width=0.7\linewidth]{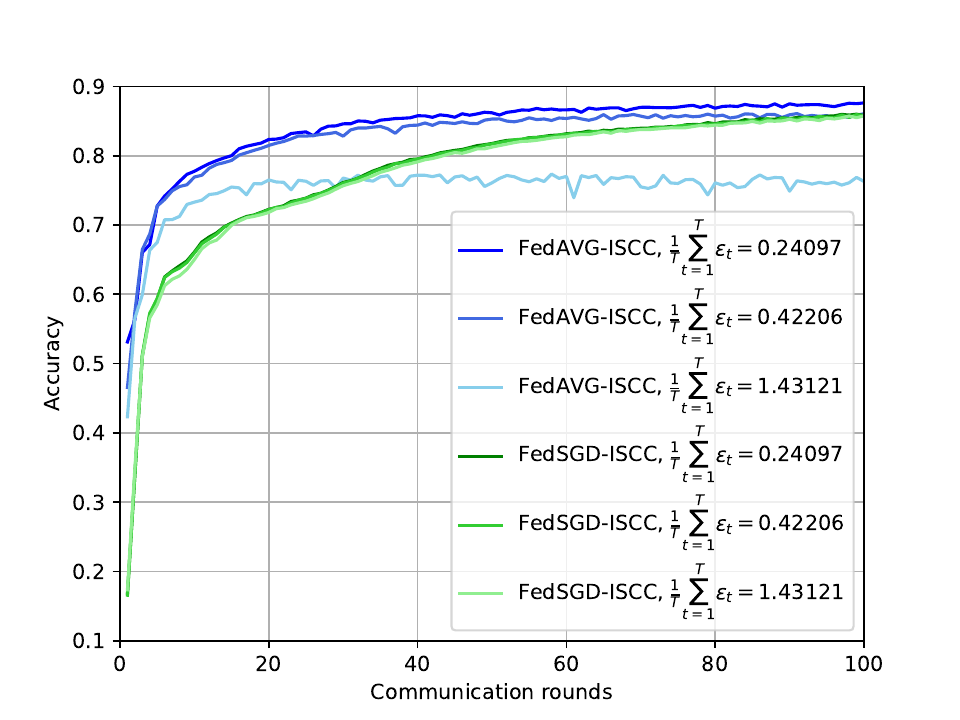}
  }%

  \caption{Performance comparison between FedAVG-ISCC and FedSGD-ISCC under different communication errors.}
  \label{CommError}
 \end{center}
\end{figure}

In Fig. \ref{CommError}, we present the accuracy of FedAVG-ISCC and FedSGD-ISCC under average different average communication errors. It is observed that the FedAVG-ISCC exhibits better performance than the FedSGD-ISCC when the average communication errors are low. However, the FedAVG-ISCC experiences severe degradation, and even worse than FedSGD-ISCC, while the FedSGD-ISCC basically remain the same performance when the average communication errors increase. This verifies our analysis that the FedSGD-ISCC is more robust to the communication errors.

\begin{figure}[]
 \begin{center}
  \subfigure[MNIST dataset]{
    \includegraphics[width=0.7\linewidth]{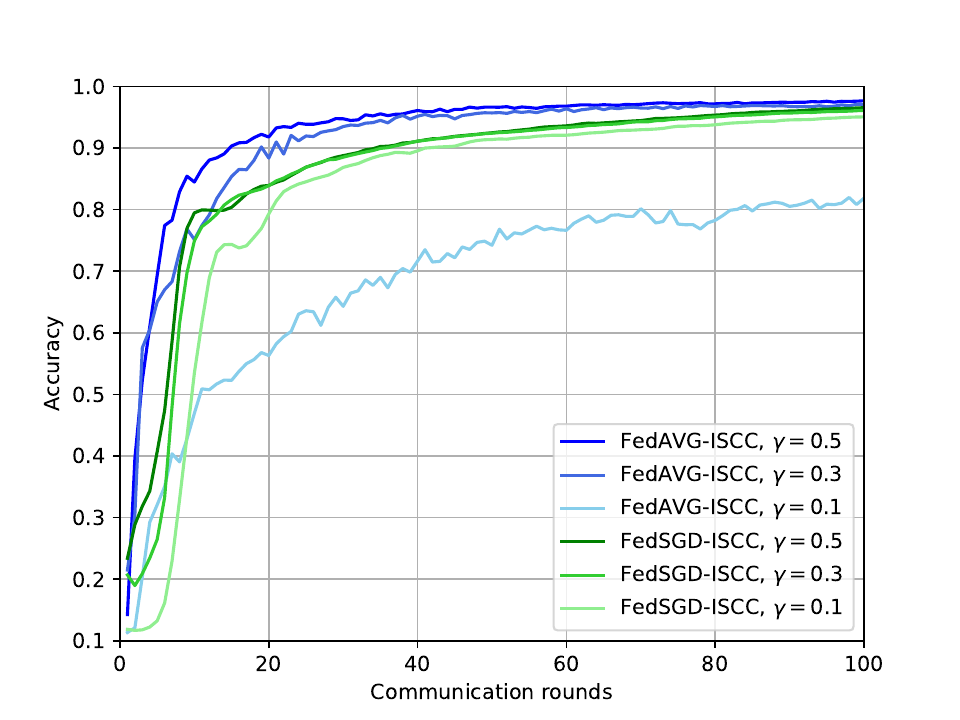}
  }%
  
  \subfigure[FMNIST dataset]{
    \includegraphics[width=0.7\linewidth]{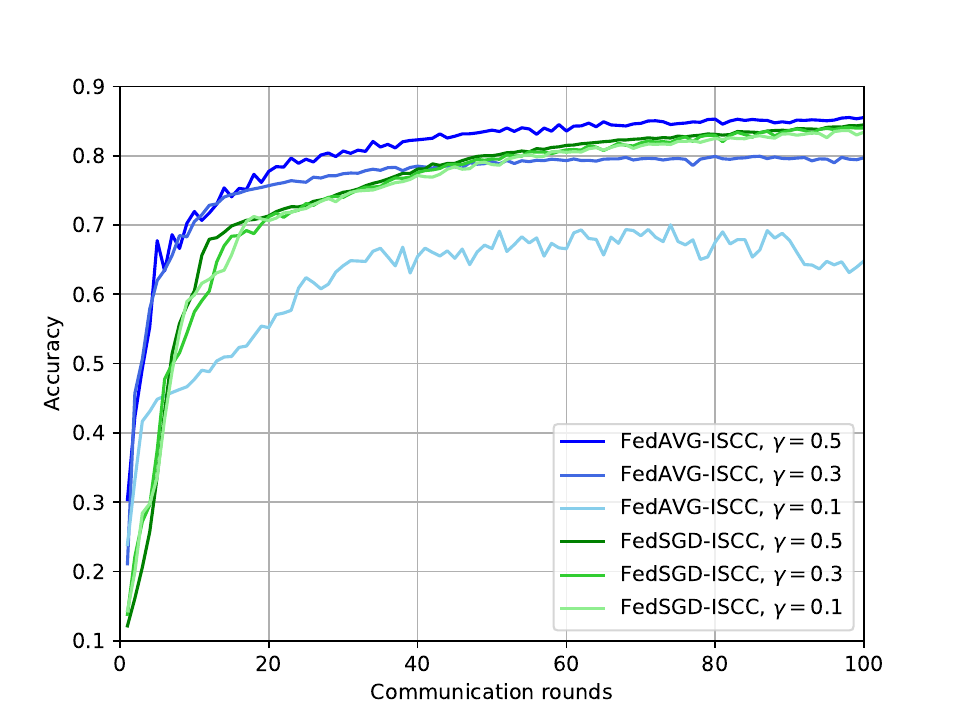}
  }%

  \caption{Performance comparison between FedAVG-ISCC and FedSGD-ISCC under different Non-IID settings.}
  \label{NonIID}
 \end{center}
\end{figure}

In Fig. \ref{NonIID}, we illustrate the accuracy of FedAVG-ISCC and FedSGD-ISCC under different Non-IID settings. The parameter $\gamma$ indicates the degree of Non-IID setting, with lower values of $\gamma$ corresponding to greater Non-IID setting. As shown in Fig. \ref{NonIID}, FedAVG-ISCC outperforms FedSGD-ISCC, similar to the IID settings, when $\gamma$ is large enough. However, as $\gamma$ decreases, the convergence of FedSGD-ISCC is only slightly affected, while the convergence of FedAVG-ISCC deteriorates significantly, ultimately performing worse than FedSGD-ISCC. This is consistent with our analysis results, indicating that FedSGD-ISCC is is more robust to Non-IID data.

To illustrate the effectiveness of our proposed FL-ISCC framework on enhancing training efficiency, we compare the accuracy versus both latency and energy consumption between FL-ISCC and OTA-FL using the MNIST dataset, as shown in \ref{Effectiveness}. The simulation results clearly indicate that both FedAVG-ISCC and FedSGD-ISCC outperform OTA-FedAVG and OTA-FedSGD, respectively, in terms of both latency and energy consumption. Therefore, the proposed FL-ISCC framework holds significant potential for enhancing training efficiency in the upcoming 6G era of intelligent applications.

\begin{figure}[]
 \begin{center}
  \subfigure[Latency]{
    \includegraphics[width=0.7\linewidth]{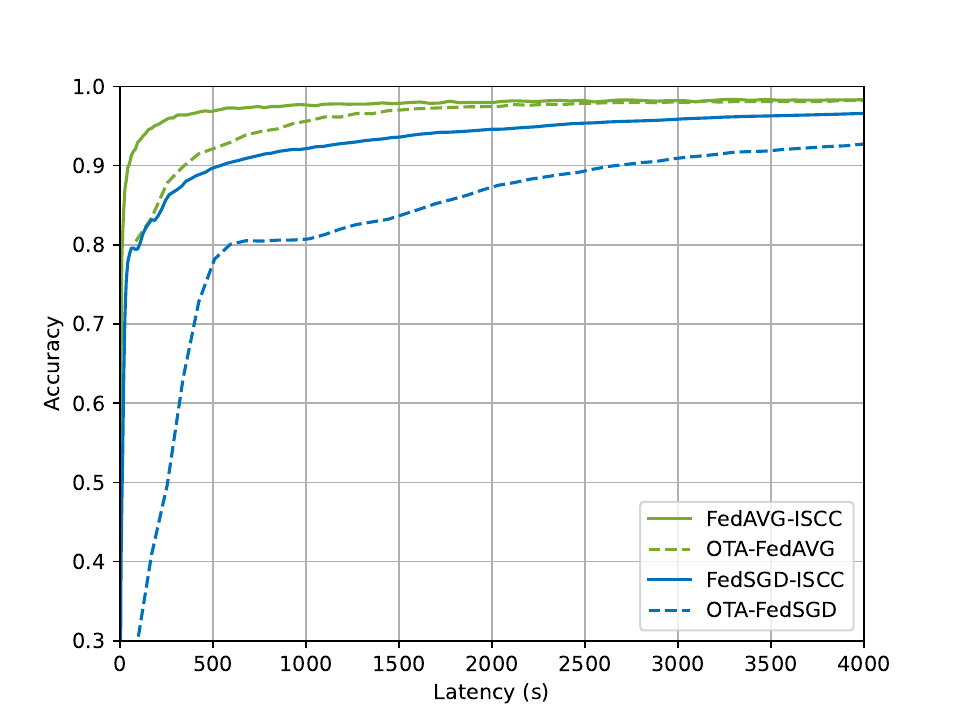}
  }%
  
  \subfigure[Energy consumption]{
    \includegraphics[width=0.7\linewidth]{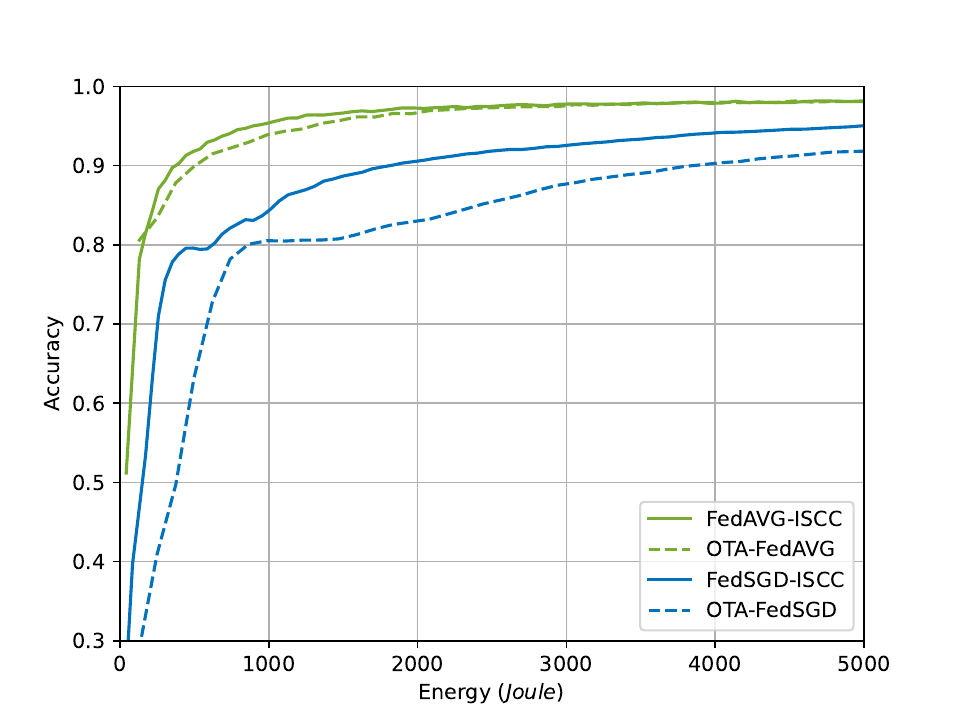}
  }%

  \caption{The effectiveness of our proposed FL-ISCC framework.}
  \label{Effectiveness}
 \end{center}
\end{figure}

\section{Conclusion}
In this work, we investigated an FL-ISCC framework that integrates sample collection, local training, and model exchange and aggregation. Within this framework, we implemented both FedAVG-ISCC and FedSGD-ISCC algorithms.Our empirical results highlighted the significant potential of the FL-ISCC framework in enhancing training efficiency, particularly in reducing latency and energy consumption in FL. 
We conducted a theoretical analysis and comparison between FedAVG-ISCC and FedSGD-ISCC, revealing that both sample collection and communication errors negatively impacted their performance. While FedAVG-ISCC significantly outperformed FedSGD-ISCC in terms of convergence rate under IID data conditions, FedSGD-ISCC exhibited greater robustness under Non-IID data, where FedAVG-ISCC experienced severe performance degradation as data heterogeneity increased. Additionally, FedSGD-ISCC proved more resilient to communication errors, whereas FedAVG-ISCC suffered notable performance degradation with increasing communication errors. Extensive simulations validated the effectiveness of our proposed framework and corroborated our theoretical findings.


%

\appendices
 \section{Proof of Lemma 1} \label{AppenA}
 According to the definition of local loss function in \eqref{localLoss}, we have
 \begin{small}
     
 \begin{equation}\label{lossLocal}
	 	\begin{split}
		 	F(\mathbf{w}^{n}_{t-1}; \mathcal{S}^{n}_{t})  = \frac{1}{S^{n}_{t}}\left[\sum_{\left(\mathbf{x}_{j}, y_{j}\right) \in \mathcal{S}^{n}_{t-1}} f\left(\mathbf{w}^{n}_{t-1}, \left(\mathbf{x}_{j}, y_{j}\right)\right) \right. \\
			\left. + \sum_{\left(\mathbf{x}_{j}, y_{j}\right) \in \mathcal{D}^{n}_{t}} f\left(\mathbf{w}^{n}_{t-1}, \left(\mathbf{x}_{j}, y_{j}\right)\right) \right]\\
		 	= \frac{S^{n}_{t-1}}{S^{n}_{t}}F(\mathbf{w}^{n}_{t-1}; \mathcal{S}^{n}_{t-1}) + \frac{D^{t}_n}{S^{n}_{t}} F(\mathbf{w}^{n}_{t-1}; \mathcal{D}^{n}_{t}).
		 	\end{split}
	 \end{equation}
 \end{small}

 As a result, the global loss function can be further rewritten as
 \begin{equation}\label{eq2}
    \begin{split}
            &\sum_{n=1}^{N} \rho^{n} F\left(\mathbf{w}^{n}_{t-1}; \mathcal{S}^{n}_{t}\right) = \frac{1}{S_{t}} \sum_{n=1}^{N} S^{n}_{t} F(\mathbf{w}^{n}_{t-1}; \mathcal{S}^{n}_{t})\\
            & = \frac{1}{S_{t}}  \! \! \sum_{n=1}^{N} \! \! \left( \! \! \frac{S^{n}_{t-1} F\left(\mathbf{w}^{n}_{t-1}; \mathcal{S}^{n}_{t-1}\right)}{S_{t-1}} S_{t-1} \! \! +  \! \! \frac{D^{n}_{t} F\left(\mathbf{w}^{n}_{t-1};\mathcal{D}^{n}_{t};\right)}{D_t} D_t \! \!\right)\\
            & = \frac{S_{t-1}}{S_{t}} \sum_{n=1}^{N} \Bar{\rho}^{n} \nabla F(\mathbf{w}^{n}_{t-1}; \mathcal{S}^{n}_{t-1})  \! \! +  \! \! \frac{D_t}{S_{t}} \sum_{n=1}^{N} \Tilde{\rho}^{n} \nabla F(\mathbf{w}^{n}_{t-1}; \mathcal{D}^{n}_t),
        \end{split}
 \end{equation}

Taking derivative with respect to $\mathbf{w}_{t-1}$ over both sides of \eqref{eq2}, Lemma \ref{lemma1} can be obtained. This ends the proof.


\section{Proof of Lemma 2}\label{AppenB}
To facilitate the proof, we introduce the following auxiliary variables
\begin{equation}
	\text { Averaged Mini-batch Gradient: } \quad \Bar{\mathbf{g}}^{n}_{t}
	=  	\frac{1}{\tau_{}}\sum_{i = 1}^{\tau_{}}\nabla F\left( \mathbf{w}^{n}_{t-1, i}; \xi \right),  
\end{equation}

\begin{equation}
	\text {Averaged Full-batch Gradient: } \quad \Bar{\mathbf{h}}^{n}_{t}
	=  	\frac{1}{\tau_{}}\sum_{i = 1}^{\tau_{}}\nabla F\left( \mathbf{w}^{n}_{t-1, i}; \mathcal{S}_{t} \right). 
\end{equation}

In FedAVG-ISCC, all devices transmit local model to edge server via over-the-air computation technique to aggregate global model. Therefore, according to \eqref{fedavg_aggre} and \eqref{commerr}, the update of global model  between two consecutive adjacent rounds is given by
\begin{equation}
	\begin{split}
		\mathbf{w}^{n}_{t} - \mathbf{w}^{n}_{t-1} 
		= \boldsymbol{\varepsilon}_t  -\eta \tau_{} \sum_{n=1}^{N}\rho^{n} \Bar{\mathbf{g}}^{n}_{t}.
	\end{split}
\end{equation}


According to \eqref{Lsmooth} in Assumption 1, the improvement on the global loss between two rounds follows that
\begin{equation} \label{Lsmooth1}
	\begin{split}
		F\left(\mathbf{w}_{t}; \mathcal{S}_{t}\right) - F\left(\mathbf{w}_{t-1};\mathcal{S}_{t-1}\right)  \leq  \underbrace{ \frac{L }{2}\left\| \boldsymbol{\varepsilon}_t  -\eta \tau_{} \sum_{n=1}^{N}\rho^{n} \Bar{\mathbf{g}}^{n}_{t} \right\|^2}_{B_1}\\
		+ \underbrace{ \left\langle \nabla F(\mathbf{w}_{t-1};\mathcal{S}_{t-1}), \boldsymbol{\varepsilon}_t - \eta \tau_{} \sum_{n=1}^{N}\rho^{n} \Bar{\mathbf{g}}^{n}_{t} \right\rangle}_{A_1} 
	\end{split}
\end{equation}

Now we aim to find the upper bound for $A_1$ and $B_1$, respectively. Specifically, for $A_1$, we have
\begin{equation}
	\begin{split}
		A_1  \! \! =   \! \!  \left\langle \nabla F(\mathbf{w}_{t-1};\mathcal{S}_{t-1}), \boldsymbol{\varepsilon}_t \right\rangle \! \! -  \! \!  \eta \tau_{} \left\langle  \! \!  \nabla F(\mathbf{w}_{t-1};\mathcal{S}_{t-1}), \sum_{n=1}^{N}\rho^{n}\Bar{\mathbf{h}}^{n}_{t}  \! \!  \right\rangle   \\
		 + \eta \tau_{} \left\langle \nabla F(\mathbf{w}_{t-1};\mathcal{S}_{t-1}), \sum_{n=1}^{N}\rho^{n} \left( \Bar{\mathbf{h}}^{n}_{t} - \Bar{\mathbf{g}}^{n}_{t} \right) \right\rangle \\
		\overset{(a)}{=} \left\langle \nabla F(\mathbf{w}_{t-1};\mathcal{S}_{t-1}), \boldsymbol{\varepsilon}_t \right\rangle   \! \! -   \! \!  \eta \tau_{} \left\langle \nabla F(\mathbf{w}_{t-1};\mathcal{S}_{t-1}), \sum_{n=1}^{N}\rho^{n}\Bar{\mathbf{h}}^{n}_{t} \right\rangle  
	\end{split}
\end{equation}
where $(a)$ comes from the fact that $\mathbb{E} \left( \Bar{\mathbf{g}}^{n}_{t} - \Bar{\mathbf{h}}^{n}_{t}\right) = 0$. 
Similarly, $B_1$ is bounded by
\begin{equation}
	\begin{split}
		B_1 &= \frac{L }{2}\left\|\boldsymbol{\varepsilon}_t - \eta \tau^{}_{} \sum_{n=1}^{N}\rho^{n} \Bar{\mathbf{g}}^{n}_{t} - \eta \tau^{}_{} \sum_{n=1}^{N}\rho^{n}\Bar{\mathbf{h}}^{n}_{t} + \eta \tau^{}_{}\sum_{n=1}^{N}\rho^{n}\Bar{\mathbf{h}}^{n}_{t} \right\|^2 \\
		& \overset{(b)}{\leq} L\eta^2 \tau^{2}_{}\left\| \sum_{n=1}^{N}\rho^{n} \left( \Bar{\mathbf{g}}^{n}_{t} - \Bar{\mathbf{h}}^{n}_{t} \right) \right\|^2  +  L \left\| \boldsymbol{\varepsilon}_t - \eta \tau^{}_{} \sum_{n=1}^{N}\rho^{n} \Bar{\mathbf{h}}^{n}_{t} \right\|^2 \\
		& \overset{(c)}{=}  L\eta^2 \tau^{2}_{} \sum_{n=1}^{N} \left( \rho^{n}\right)^2 \left\| \Bar{\mathbf{g}}^{n}_{t} -\Bar{\mathbf{h}}^{n}_{t} \right\|^2  +  L \left\| \boldsymbol{\varepsilon}_t - \eta \tau^{}_{} \sum_{n=1}^{N}\rho^{n} \Bar{\mathbf{h}}^{n}_{t} \right\|^2 \\
            & \overset{(d)}{\leq } L \tau^{}_{} \eta^2 \sigma^2 \sum_{n=1}^{N} \left( \rho^{n}\right)^2 + L \left\| \boldsymbol{\varepsilon}_t \right\|^2 +  L\eta^2 \tau^{2}_{}\left\| \sum_{n=1}^{N}\rho^{n} \Bar{\mathbf{h}}^{n}_{t} \right\|^2 \\
            &~~~ - 2L\eta \tau^{}_{} \left\langle  \sum_{n=1}^{N}\rho^{n} \Bar{\mathbf{h}}^{n}_{t} , \boldsymbol{\varepsilon}_t \right\rangle 
	\end{split}
\end{equation}
where $(b)$ comes from the fact that $\| a+b\|^2 \leq 2 \| a \|^2 +  2 \| b \|^2$. $(c)$ is achieved due to the fact that clients are independent to each other, i.e., $\mathbb{E} \left\langle \Bar{\mathbf{g}}^{i}_{t} -\Bar{\mathbf{h}}^{i}_{t}, \Bar{\mathbf{g}}^{j}_{t} -\Bar{\mathbf{h}}^{j}_{t} \right\rangle = 0, \forall i \neq j$. (d) comes from \textbf{Assumption 3}. 

As a result, when $0 \leq L\eta^{}_{}  \tau^{}_{}  \leq \frac{1}{2}$, \eqref{Lsmooth1} is reformulated as 
\begin{equation} \label{Lsmooth2}
    \begin{split}
        &F\left(\mathbf{w}_{t}; \mathcal{S}_{t}\right) - F\left(\mathbf{w}_{t-1};\mathcal{S}_{t-1}\right) \leq \left(1-2L\eta \tau^{}_{}  \right)  \left\langle\sum_{n=1}^{N}\rho^{n} \Bar{\mathbf{h}}^{n}_{t}, \boldsymbol{\varepsilon}_t \right\rangle\\ %
        & + \left\langle \nabla F(\mathbf{w}_{t-1};\mathcal{S}_{t-1}) - \sum_{n=1}^{N}\rho^{n} \Bar{\mathbf{h}}^{n}_{t}, \boldsymbol{\varepsilon}_t \right\rangle + L \tau^{}_{} \eta^2 \sigma^2 \sum_{n=1}^{N} \left( \rho^{n}\right)^2  \\
        & + \! \! L\eta^{2}_{} \tau^{2}_{}\left\| \sum_{n=1}^{N}\rho^{n} \Bar{\mathbf{h}}^{n}_{t} \right\|^2  \! \! \!  -  \! \! \!  \eta \tau_{} \left\langle  \! \!  \nabla F(\mathbf{w}_{t-1};\mathcal{S}_{t-1}), \sum_{n=1}^{N}\rho^{n}\Bar{\mathbf{h}}^{n}_{t}  \! \! \! \right\rangle  \! \! \!  +  \! \! \!  L \! \left\| \boldsymbol{\varepsilon}_t \right\|^2  \\
        & \overset{(e)}{\leq} \! \! \! -\frac{\tau^{}_{} \eta}{2}  \! \! \! \left[\left\| \nabla F(\mathbf{w}_{t-1};\mathcal{S}_{t-1}) \right\|^2  \! \!  -  \! \! \left\| \nabla F(\mathbf{w}_{t-1};\mathcal{S}_{t-1}) \! \!  - \! \! \sum_{n=1}^{N}\rho^{n} \Bar{\mathbf{h}}^{n}_{t} \right\|^2  \right. \\
	&\left. + \left\| \sum_{n=1}^{N}\rho^{n} \Bar{\mathbf{h}}^{n}_{t} \right\|^2  \right]  +  \frac{ \left( 1 - 2L\eta^{}_{}  \tau^{}_{} \right) \eta^{}_{}  \tau^{}_{}   }{2} \left\| \sum_{n=1}^{N}\rho^{n} \Bar{\mathbf{h}}^{n}_{t} \right\|^2  + \frac{\left\| \boldsymbol{\varepsilon}_t \right\|^2}{2\eta \tau^{}_{}}  \\
        & + \frac{\eta \tau^{}_{}}{2} \left\| \nabla F(\mathbf{w}_{t-1};\mathcal{S}_{t-1}) -\sum_{n=1}^{N}\rho^{n} \Bar{\mathbf{h}}^{n}_{t} \right\|^2 + \frac{  \left( 1 - 2L\eta^{}_{}  \tau^{}_{} \right) }{2 \eta^{}_{}  \tau^{}_{}} \left\| \boldsymbol{\varepsilon}_t \right\|^2 \\
        & + L \tau^{}_{} \eta^2 \sigma^2 \sum_{n=1}^{N} \left( \rho^{n}\right)^2 + L \left\| \boldsymbol{\varepsilon}_t \right\|^2 +  L\eta^2 \tau^{2}_{}\left\| \sum_{n=1}^{N}\rho^{n} \Bar{\mathbf{h}}^{n}_{t} \right\|^2 \\
        & \overset{}{=}  -\frac{\tau^{}_{} \eta}{2} \left\| \nabla F(\mathbf{w}_{t-1};\mathcal{S}_{t-1}) \right\|^2 + L \tau^{}_{} \eta^{2}_{} \sigma^2 \sum_{n=1}^{N} \left( \rho^{n}\right)^2 +  \frac{ \left\| \boldsymbol{\varepsilon}_t \right\|^2  }{ \eta^{}_{} \tau^{}_{}}\\
        &  + \tau^{}_{} \eta^{}_{} \underbrace{ \left\| \nabla F(\mathbf{w}_{t-1};\mathcal{S}_{t-1}) - \sum_{n=1}^{N}\rho^{n} \Bar{\mathbf{h}}^{n}_{t} \right\|^2}_{C_1} 
    \end{split}
\end{equation}
where (e) derives from the fact that $2\left\langle a,b\right\rangle = \|a\|^2 + \|b\|^2 -\|a-b\|^2$, and $\left\langle \boldsymbol{a}, \boldsymbol{b} \right\rangle \leq \frac{x\left\| \boldsymbol{a}\right\| ^2}{2} + \frac{\left\| \boldsymbol{b}\right\| ^2}{2x}$ with $x = \tau^{}_{} \eta \geq 0$. 

To proof Lemma 2, we first derive the improvement at the first communication round, and then extend to the rest communication rounds.

\textit{1) Improvement in the first communication round:} 
The ML model is updated based on initialization $\mathbf{w}_{0}$ over the new sensed dataset $\mathcal{D}_1$ in the current round.
Therefore, using the Lipschitz-smooth property and Jensen inequality,  $C_1$ can be expressed as:
\begin{equation}\label{C_1}
	\begin{split}
		C_1 &= \left\|  \nabla F(\mathbf{w}_{0};\mathcal{D}_{1}) -  \frac{1}{\tau_{}}  \sum_{n=1}^{N}\rho^{n} \sum_{i = 1}^{\tau_{}}\nabla F\left( \mathbf{w}^{n}_{0,i}; \mathcal{D}_{1} \right) \right\|^2\\
		& \leq \frac{1}{\tau_{}} \sum_{n=1}^{N}\rho^{n} \sum_{i = 1}^{\tau_{}}  \left\|  \nabla F(\mathbf{w}_{0};\mathcal{D}_{1}) -  \nabla F\left( \mathbf{w}^{n}_{0,i}; \mathcal{D}_{1} \right) \right\|^2 \\ 
		& \leq \frac{L^2}{\tau_{}} \sum_{n=1}^{N}\rho^{n} \sum_{i = 1}^{\tau_{}} \underbrace{ \left\|  \mathbf{w}_{0} -   \mathbf{w}^{n}_{0,i} \right\|^2}_{D_1}. 
	\end{split}
\end{equation}
Furthermore, using the fact that $\| a+b\|^2 \leq 2 \| a \|^2 +  2 \| b \|^2$, for $\forall i$ in \eqref{C_1}, we have
\begin{equation} \label{D_1}
	\begin{split}
		& D_1 = \eta^2 \left\| \sum_{j = 1}^{i} \nabla F\left( \mathbf{w}^{n}_{0,i}; \xi \right) \right\|^2\\
		& \leq 2 \eta^2 \left\| \sum_{j = 1}^{i} \left[\nabla F\left( \mathbf{w}^{n}_{0,i}; \xi \right) - \nabla F\left( \mathbf{w}^{n}_{0,i}; \mathcal{D}^{n}_{1} \right)\right] \right\|^2 \\
            & ~~~ + 2\eta^2 \left\| \sum_{j = 1}^{i} \nabla F\left( \mathbf{w}^{n}_{0,i}; \mathcal{D}^{n}_{1} \right) \right\|^2 \\
		& \overset{(f)}{\leq} 2 i \eta^2 \sum_{j = 1}^{i} \sigma^2 +  2\eta^2 i \sum_{j = 1}^{i} \left\|\nabla F\left( \mathbf{w}^{n}_{0,j}; \mathcal{D}^{n}_{1} \right) \right\|^2 \\
		& \leq 2 i \eta^2 \sigma^2 +  2\eta^2 i \sum_{j = 1}^{\tau_{} } \left\|\nabla F\left( \mathbf{w}^{n}_{0,j}; \mathcal{D}^{n}_{1} \right) \right\|^2
	\end{split}
\end{equation}
where $(f)$ is derived from Cauchy–Schwarz inequality and \textbf{Assumption 3}.  Using the equation $\sum_{i = 1}^{\tau_{}} i = \frac{\tau_{}\left(\tau_{} - 1 \right)}{2}$, we obtain
\begin{equation}
	\begin{split}
		& \sum_{i = 1}^{\tau_{}}  \left\|  \mathbf{w}^{n}_{0} -   \mathbf{w}^{n}_{0,i} \right\|^2 \leq \eta^2 \tau_{}\left( \tau_{} -1\right) \left(  \sigma^2 + \sum_{j = 1}^{\tau_{} } \left\|\nabla F\left( \mathbf{w}^{n}_{0,j}; \mathcal{D}^{n}_{1} \right) \right\|^2 \right) \\
		& \leq \eta^2 \tau_{}\left( \tau_{} -1\right) \left( \sigma^2 + 2 \sum_{j = 1}^{\tau_{} } \left\|  \nabla F(\mathbf{w}^{n}_{0};\mathcal{D}^{n}_{1}) \right\|^2 \right) \\
            & + 2 \eta^2 \tau_{}\left( \tau_{} -1\right) \sum_{j = 1}^{\tau_{} } \left( \left\|\nabla F\left( \mathbf{w}^{n}_{0,j}; \mathcal{D}^{n}_{1} \right)  - \nabla F(\mathbf{w}^{n}_{0};\mathcal{D}^{n}_{1}) \right\|^2  \right)  \\
		& \overset{(g)}{\leq} \eta^2 \tau_{} \left( \tau_{} -1\right) \left( \sigma^2 +  2 \sum_{j = 1}^{\tau_{} } \left\|  \nabla F(\mathbf{w}^{n}_{0};\mathcal{D}^{n}_{1}) \right\|^2 \right)  \\
            & + 2 L^2 \eta^2 \tau_{}\left( \tau_{} -1\right) \sum_{j = 1}^{\tau_{} }  \left\| \mathbf{w}^{n}_{0,j}  - \mathbf{w}^{n}_{0} \right\|^2 
	\end{split}
\end{equation}
where $(g)$ follows Lipschitz-smooth property.
After rearranging, we have
\begin{equation}\label{}
	\begin{split}
		& \sum_{i = 1}^{\tau_{}}  \left\|  \mathbf{w}^{n}_{0} -   \mathbf{w}^{n}_{0,i} \right\|^2 \leq \\
  &\frac{\eta^2 \sigma^2 \tau_{}\left( \tau_{} -1\right)}{1 - 2 L^2 \eta^2 \tau_{}\left( \tau_{} -1\right)} + \frac{ 2 \eta^2  \tau^{2}_{}\left( \tau_{} -1\right)}{1 - 2 L^2 \eta^2 \tau_{}\left( \tau_{} -1\right)}\left\|  \nabla F(\mathbf{w}^{n}_{0};\mathcal{D}^{n}_{1}) \right\|^2\\
	\end{split}
\end{equation}
Therefore, \eqref{C_1} is bounded by
\begin{equation}\label{C_1_1}
	\begin{split}
		C_1 & \leq   \frac{L^2\eta^2 \sigma^2 \left( \tau_{} -1\right)}{1 - 2 L^2 \eta^2 \tau_{}\left( \tau_{} -1\right)} + \frac{ 2L^2 \eta^2  \tau^{}_{}\left( \tau_{} -1\right)}{1 - 2 L^2 \eta^2 \tau_{}\left( \tau_{} -1\right)}\left\|  \nabla F(\mathbf{w}^{n}_{0};\mathcal{D}^{n}_{1}) \right\|^2\\
		& = \frac{ L^2\eta^2 \sigma^2 \left( \tau_{} -1\right)}{1-A} + \frac{A}{1-A} \sum_{n=1}^{N}\rho^{n}  \left\|  \nabla F(\mathbf{w}^{n}_{0};\mathcal{D}^{n}_{1}) \right\|^2,
	\end{split}
\end{equation}
where  $A = 2 L^2 \eta^2 \tau_{}\left( \tau_{} -1\right)$. Plug $\eqref{C_1_1}$ back into  $\eqref{Lsmooth2}$, we have
\begin{equation} \label{Lsmooth2_1_1}
	\begin{split}
		&\mathbb{E} \left(F\left(\mathbf{w}_{1}; \mathcal{S}_{1}\right) - F\left(\mathbf{w}_{0};\mathcal{S}_{0}\right)\right) \leq \\ 
            &- \frac{\eta\tau_{}}{2}\left\| \nabla F(\mathbf{w}_{0};\mathcal{D}_{1}) \right\|^2  + L \tau^{}_{} \eta^2 \sigma^2 \sum_{n=1}^{N} \left( \rho^{n}\right)^2 +  \frac{ 1 }{ \eta^{}_{} \tau^{}_{}} \left\| \boldsymbol{\varepsilon}_{1} \right\|^2  \\  
		& + \eta^{}_{}\tau^{}_{} \left( \frac{L^2 \eta^2 \sigma^2  \left( \tau_{} -1\right)}{1 - A} + \frac{ A}{1 - A} \sum_{n=1}^{N}\rho^{n} \left\|  \nabla F(\mathbf{w}^{n}_{0};\mathcal{D}^{n}_{1}) \right\|^2 \right) \\
		& \overset{(h)}{\leq} -  \frac{\eta\tau_{}}{2} \left( 1 - \frac{2 \alpha^2 A}{1 - A }\right) \left\| \nabla F(\mathbf{w}_{0};\mathcal{D}_{1}) \right\|^2 \! \!  + L \tau^{}_{} \eta^2 \sigma^2 \sum_{n=1}^{N} \left( \rho^{n}\right)^2 \\ 
		&  + \eta^{}_{} \tau^{}_{} \frac{L^2 \eta^2 \sigma^2  \left( \tau_{} -1\right)}{1 - A}  + \frac{ \beta^2 A  \eta^{}_{}\tau^{}_{} }{1 - A}  +  \frac{ 1 }{ \eta^{}_{} \tau^{}_{}} \left\| \boldsymbol{\varepsilon}_t \right\|^2 \\ 
		& \overset{(i)}{\leq} -  \frac{\eta\tau_{}}{4} \left( 2 -  \alpha^2  \right) \left\| \nabla F(\mathbf{w}_{0};\mathcal{D}_{1}) \right\|^2  + L \tau^{}_{} \eta^2 \sigma^2 \sum_{n=1}^{N} \left( \rho^{n}\right)^2 \\
		& + \eta^{}_{} \tau^{}_{} \frac{ 5L^2 \eta^2 \sigma^2  \left( \tau_{} -1\right)}{4}+ \frac{ \beta^2  \eta^{}_{}\tau^{}_{} }{4} +  \frac{ 1 }{ \eta^{}_{} \tau^{}_{}} \left\| \boldsymbol{\varepsilon}_t \right\|^2.  \\ 
	\end{split}
\end{equation}
where $(h)$ is achieved due to \textbf{Assumption 3} and  \textbf{Assumption 4}. $(i)$ is derived from $A = 2 L^2 \eta^2 \tau_{}\left( \tau_{} -1\right) \leq \frac{1}{5}$. 


\textit{2) Improvement in the rest communication rounds:} For the rest communication rounds, the ML model is updated based on both the accumulative dataset $\mathcal{S}_{t-1}$ and the newly sensed dataset $\mathcal{D}_{t}$. According to \textbf{Lemma 1 }, $C_1$ can be expressed as
\begin{equation}\label{C_1_2}
	\begin{split} 
		& C_1 \overset{}{\leq} \frac{1}{\tau_{}}\sum_{i = 1}^{\tau_{}}  \left\|  \nabla F(\mathbf{w}^{}_{t-1};\mathcal{S}^{}_{t-1}) - \sum_{n=1}^{N}\rho^{n}  \nabla F\left( \mathbf{w}^{n}_{t-1,i}; \mathcal{S}^{n}_{t} \right) \right\|^2 \\
		& \overset{(j)}{\leq} \! \! \frac{2}{\tau_{}}\sum_{i = 1}^{\tau_{}} \! \! \left( \! \!  \frac{S^{2}_{t-1}}{S^{2}_{t}} \! \! \sum_{n=1}^{N} \Bar{\rho }^{n}  \left\| \nabla F(\mathbf{w}^{}_{t-1};\mathcal{S}^{}_{t-1}) \! \! - \! \! \nabla F(\mathbf{w}^{n}_{t-1, i}; \mathcal{S}^{n}_{t-1})  \right\|^2  \right. \\
		& \left. - \sum_{n=1}^{N} \Tilde{\rho }^{n}  \left\| \frac{D_t}{S_{t}} \nabla F(\mathbf{w}^{}_{t-1};\mathcal{S}^{}_{t-1}) - \frac{D_t}{S_{t}}\nabla F(\mathbf{w}^{n}_{t-1, i}; \mathcal{D}^{n}_t) \right\|^2 \right)  \\
		& \overset{(k)}{\leq} \frac{2L^2}{\tau_{}}\sum_{i = 1}^{\tau_{}} \left(  \left( \frac{S_{t-1}}{S_{t}}\right)^2 \sum_{n=1}^{N} \Bar{\rho }^{n} \underbrace{ \left\|  \mathbf{w}_{t-1} -   \mathbf{w}^{n}_{t-1,i} \right\|^2}_{D_2}   \right. \\
		& \left. +  \left( \frac{D_t}{S_{t}} \right)^2  \sum_{n=1}^{N} \Tilde{\rho }^{n} \underbrace{ \left\|  \mathbf{w}_{t-1} -   \Bar{\mathbf{w}}^{n}_{t-1,i} \right\|^2}_{D_3}  \right),
	\end{split}
\end{equation}
where $(j)$ and $(k)$ come from the fact that $\| a+b\|^2 \leq 2 \| a \|^2 +  2 \| b \|^2$, and Lipschitz-smooth property, respectively. $\mathbf{w}^{n}_{t-1,i}$ and $\Bar{\mathbf{w}}^{n}_{t-1,i}$ are the updated model after $i$-th local step based on the datasets $\mathcal{S}^{n}_{t-1}$ and $\mathcal{D}_t$, respectively. Similar to \eqref{D_1}, for $\forall i$ in \eqref{C_1_2}, we have
\begin{equation} \label{D_2}
	\begin{split}
		& D_2 = 2\eta^2 \left\| \sum_{j = 1}^{i} \nabla F\left( \mathbf{w}^{n}_{t-1,j}; \mathcal{S}^{n}_{t-1} \right) \right\|^2 \\
		& + 2 \eta^2 \left\| \sum_{j = 1}^{i} \left[\nabla F\left( \mathbf{w}^{n}_{t-1,j}; \xi \right) - \nabla F\left( \mathbf{w}^{n}_{t-1,j};\mathcal{S}^{n}_{t-1} \right)\right] \right\|^2 \\
		& \leq 2 i \eta^2 \sigma^2 +  2\eta^2 i \sum_{j = 1}^{\tau_{} } \left\|\nabla F\left( \mathbf{w}^{n}_{t-1,j};\mathcal{S}^{n}_{t-1} \right) \right\|^2. 
	\end{split}
\end{equation}
Using  the equation $\sum_{i = 1}^{\tau_{}} i = \frac{\tau_{}\left(\tau_{} - 1 \right)}{2}$, we have
\begin{equation}
	\begin{split}
		& \sum_{i = 1}^{\tau_{}}  \left\|  \mathbf{w}_{t-1} -   \mathbf{w}^{n}_{t-1,i} \right\|^2 \leq \\
  & \eta^2 \tau_{}\left( \tau_{} -1\right) \left(\sigma^2 + \sum_{j = 1}^{\tau_{} } \left\|\nabla F\left( \mathbf{w}^{n}_{t-1,j};\mathcal{S}^{n}_{t-1} \right) \right\|^2\right)  \\
		& \leq \eta^2 \tau_{}\left( \tau_{} -1\right) \left(\sigma^2 + 2 \sum_{i = 1}^{\tau_{} } \left\|\nabla F\left( \mathbf{w}^{n}_{t-1};\mathcal{S}^{n}_{t-1} \right) \right\|^2\right) \\
  &+ 2 L^2 \eta^2 \tau_{}\left( \tau_{} -1\right) \sum_{j = 1}^{\tau_{} }  \left\| \mathbf{w}^{n}_{t-1,j}  - \mathbf{w}^{n}_{t-1} \right\|^2 . 
	\end{split}
\end{equation}
After rearranging, we have
\begin{equation}
	\begin{split}
		& \sum_{i = 1}^{\tau_{}}  \left\|  \mathbf{w}_{t-1} -   \mathbf{w}^{n}_{t-1,i} \right\|^2 \leq \frac{\eta^2 \sigma^2 \tau_{}\left( \tau_{} -1\right)}{1 - 2 L^2 \eta^2 \tau_{}\left( \tau_{} -1\right)} \\
  &+ \frac{ 2 \eta^2  \tau^{2}_{}\left( \tau_{} -1\right)}{1 - 2 L^2 \eta^2 \tau_{}\left( \tau_{} -1\right)}\left\|  \nabla F(\mathbf{w}^{n}_{t-1};\mathcal{S}^{n}_{t-1}) \right\|^2.
	\end{split}
\end{equation}
Similarly, we can bound $D_3$ by the same way, which  can be presented by
\begin{equation}
	\begin{split}
		& \sum_{i = 1}^{\tau_{}}  \left\|  \mathbf{w}_{t-1} -   \Bar{\mathbf{w}}^{n}_{t-1,i} \right\|^2 \leq \frac{\eta^2 \sigma^2 \tau_{}\left( \tau_{} -1\right)}{1 - 2 L^2 \eta^2 \tau_{}\left( \tau_{} -1\right)} \\
  &+ \frac{ 2 \eta^2  \tau^{2}_{}\left( \tau_{} -1\right)}{1 - 2 L^2 \eta^2 \tau_{}\left( \tau_{} -1\right)}\left\|  \nabla F(\mathbf{w}^{n}_{t-1};\mathcal{D}^{n}_{t}) \right\|^2.
	\end{split}
\end{equation}

As a result, \eqref{C_1_2} is bounded by
\begin{equation}\label{C_1_3}
	\begin{split} 
		& C_1  \leq \frac{ A}{1 - A} \left[  \left( \frac{S_{t-1}}{S_{t}} \right)^2 \sum_{n=1}^{N} \Bar{\rho }^{n} \left\|  \nabla F(\mathbf{w}^{n}_{t-1};\mathcal{S}^{n}_{t-1}) \right\|^2 \right. \\
		& \left. + \left(  \frac{D_t}{S_{t}} \right)^2 \sum_{n=1}^{N} \Tilde{\rho }^{n} \left\|  \nabla F(\mathbf{w}^{n}_{t-1};\mathcal{D}^{n}_{t}) \right\|^2  \right] \\  
		& + \left( 1 - \frac{2S_{t-1}D_{t}}{S^{2}_{t}} \right) \frac{ L^2 \eta^2 \sigma^2\left( \tau_{} -1\right)}{1 - A}.   
	\end{split}
\end{equation}

Plug $\eqref{C_1_3}$ back into $\eqref{Lsmooth2}$, we have
\begin{equation} \label{Lsmooth2_1_2}
	\begin{split}
		&\mathbb{E} \left(F\left(\mathbf{w}_{t}; \mathcal{S}_{t}\right) - F\left(\mathbf{w}_{t-1};\mathcal{S}_{t-1}\right)\right) \overset{(l)}{\leq} \frac{ \eta^{}_{}\tau^{}_{} A \alpha^2 }{1 - A} \frac{D^{2}_{t}}{S^{2}_{t}} G_{t}   +  \frac{ 1 }{ \eta^{}_{} \tau^{}_{}} \left\| \boldsymbol{\varepsilon}_t \right\|^2 \\ 
            &  - \! \! \frac{\tau^{}_{} \eta}{2} \! \! \left( 1 -  \frac{  2 A \alpha^2 }{1 - A}  \frac{S^{2}_{t-1}}{S^{2}_{t}}  \right)  \! \!  \left\| \nabla F(\mathbf{w}_{t-1};\mathcal{S}_{t-1}) \right\|^2  \! \! + L \tau^{}_{} \eta^{2}_{} \sigma^2 \sum_{n=1}^{N} \left( \rho^{n}\right)^2 \\
            & +  \left( 1 - \frac{2S_{t-1}D_{t}}{S^{2}_{t}} \right)  \! \!  \left(  \frac{ L^2 \eta^3 \sigma^2 \tau_{} \left( \tau_{} -1\right)}{1 - A}  \! \!  +  \! \! \frac{ \eta^{}_{}\tau^{}_{} A}{1 - A} \beta^2 \right) \\ 
			& \overset{(m)}{\leq} \! \! \! \! -\frac{\tau^{}_{} \eta}{ 4 } \left( 2 - \alpha^2 \right)  \left\| \nabla F(\mathbf{w}_{t-1};\mathcal{S}_{t-1}) \right\|^2  \! \! +  \! \!  L \tau^{}_{} \eta^{2}_{} \sigma^2 \sum_{n=1}^{N} \left( \rho^{n}\right)^2  \! \!  +  \! \!  \frac{ \left\| \boldsymbol{\varepsilon}_t \right\|^2 }{ \eta^{}_{} \tau^{}_{}} \\
			& +  \left( 1+ \frac{ S^{2}_{t} }{ 4 S^{2}_{t-1} } \right)  L^2 \eta^3 \sigma^2 \tau_{} \left( \tau_{} -1\right) + \frac{ \eta^{}_{}\tau^{}_{}S^{2}_{t}  }{4S^{2}_{t-1}}  \beta^2   +  \frac{ \eta^{}_{}\tau^{}_{} \alpha^2 }{4}  \frac{D^{2}_{t}}{S^{2}_{t-1}} G_{t} .
	\end{split}
\end{equation}
where $(l)$ comes from \textbf{Assumption 2} and \textbf{Assumption 4}. $(m)$ follows the fact that $ A \leq \frac{S^{2}_{t}}{S^{2}_{t} + 4 S^{2}_{t-1}}$. 
This completes the proof.

\section{Proof of Lemma 4}\label{AppenC}
In FedSGD-ISCC, all devices transmit the gradients to edge server for gradient aggregation via over-the-air computation technique. Then, the global model is updated based on the aggregated gradient. Therefore, according to \eqref{Update_fedsgd} and \eqref{commerr}, the update of global model  between two consecutive adjacent rounds is given by
\begin{equation}
	\mathbf{w}_{t}- \mathbf{w}_{t-1} = \boldsymbol{\varepsilon}_1 - \sum_{n=1}^{N} \rho^{n} \nabla F\left( \mathbf{w}^{n}_{t-1};\mathcal{S}^{n}_{t} \right)
\end{equation}

According to \textbf{assumption 1}, the improvement on the global loss can be expressed as:
	\begin{equation}\label{Lsmooth3}
		\begin{split}
			& F\left(\mathbf{w}_{t}; \mathcal{S}_{t}\right) \!-\! F\left(\mathbf{w}_{t-1};\mathcal{S}_{t-1}\right) \leq \\ 
		&\underbrace{ \!\eta\! \left\langle \nabla F(\mathbf{w}_{t-1};\mathcal{S}_{t-1}), \! \boldsymbol{\varepsilon}_t \!-\! \sum_{n=1}^{N} \rho^{n} \nabla F\left( \mathbf{w}^{n}_{t-1};\mathcal{S}^{n}_{t} \right) \right\rangle}_{A_2} \\
  & + \underbrace{ \frac{L\eta^2}{2} \! \!\left\| - \sum_{n=1}^{N} \rho^{n} \nabla F\left( \mathbf{w}^{n}_{t-1};\mathcal{S}^{n}_{t} \right)\! +\! \boldsymbol{\varepsilon}_t \right\|^2}_{B_2}.
		\end{split}
	\end{equation} 


Now we aim to find the upper bound for $A_2$ and $B_2$, respectively. 
Specifically, for $A_2$, we have
\begin{equation}
	\begin{split}
	A_2 = &- \frac{\eta }{2} \left\| \nabla F(\mathbf{w}_{t-1};\mathcal{S}_{t-1}) \right\|^2 - \frac{\eta }{2} \left\|   \sum_{n=1}^{N} \rho^{n} \nabla F\left( \mathbf{w}^{n}_{t-1};\mathcal{S}^{n}_{t} \right) \right\|^2 \\
	&  + \frac{\eta }{2} \left\| \nabla F(\mathbf{w}_{t-1};\mathcal{S}_{t-1}) -  \sum_{n=1}^{N} \rho^{n} \nabla F\left( \mathbf{w}^{n}_{t-1};\mathcal{S}^{n}_{t} \right) \right\|^2 \\
        & + \left\langle \nabla F(\mathbf{w}_{t-1};\mathcal{S}_{t-1}),  \boldsymbol{\varepsilon}_t \right\rangle 
\end{split}
\end{equation}
For $B_2$, we have
\begin{equation}
    \begin{split}
        B_2  = & \frac{L \eta^2 }{2} \left\|   \sum_{n=1}^{N} \rho^{n} \nabla F\left( \mathbf{w}^{n}_{t-1};\mathcal{S}^{n}_{t} \right) \right\|^2 + \frac{L \eta^2 }{2} \left\|  \boldsymbol{\varepsilon}_t  \right\|^2 \\
        &- L \eta^2 \left\langle  \sum_{n=1}^{N} \rho^{n} \nabla F\left( \mathbf{w}^{n}_{t-1};\mathcal{S}^{n}_{t} \right),  \boldsymbol{\varepsilon}_t \right\rangle 
    \end{split}
\end{equation}

Similar to \eqref{Lsmooth2}, when $\eta \leq \frac{1}{L}$, the \eqref{Lsmooth3} can be further bounded by
\begin{equation}\label{Lsmooth3_1}
    \begin{split}
        &F\left(\mathbf{w}_{t}; \mathcal{S}_{t}\right) \!-\! F\left(\mathbf{w}_{t-1};\mathcal{S}_{t-1}\right)  \leq - \frac{\eta }{2} \left\| \nabla F(\mathbf{w}_{t-1};\mathcal{S}_{t-1}) \right\|^2 \\
        & + \frac{\eta }{2} \left(L \eta  -1 \right) \left\|   \sum_{n=1}^{N} \rho^{n} \nabla F\left( \mathbf{w}^{n}_{t-1};\mathcal{S}^{n}_{t} \right) \right\|^2  + \frac{L \eta^2 }{2} \left\|  \boldsymbol{\varepsilon}_t  \right\|^2 \\ 
        &+ \frac{\eta }{2} \left\| \nabla F(\mathbf{w}_{t-1};\mathcal{S}_{t-1})-  \sum_{n=1}^{N} \rho^{n} \nabla F\left( \mathbf{w}^{n}_{t-1};\mathcal{S}^{n}_{t} \right) \right\|^2  \\
         &+ \eta \left\langle \nabla F(\mathbf{w}_{t-1};\mathcal{S}_{t-1}) - \sum_{n=1}^{N} \rho^{n} \nabla F\left( \mathbf{w}^{n}_{t-1};\mathcal{S}^{n}_{t} \right),  \boldsymbol{\varepsilon}_t \right\rangle \\
         & + \eta \left( 1 - L\eta \right) \left\langle   \sum_{n=1}^{N} \rho^{n} \nabla F\left( \mathbf{w}^{n}_{t-1};\mathcal{S}^{n}_{t} \right),  \boldsymbol{\varepsilon}_t \right\rangle\\
         &\overset{}{\leq}  - \frac{\eta }{2} \left\| \nabla F(\mathbf{w}_{t-1};\mathcal{S}_{t-1}) \right\|^2 + \eta \left\|  \boldsymbol{\varepsilon}_t  \right\|^2 \\
         &  + \underbrace{ \eta \left\| \nabla F(\mathbf{w}_{t-1};\mathcal{S}_{t-1})-  \sum_{n=1}^{N} \rho^{n} \nabla F\left( \mathbf{w}^{n}_{t-1};\mathcal{S}^{n}_{t} \right) \right\|^2}_{C_2}.
    \end{split}
\end{equation}

Similar to the proof of lemma 2, we first derive the improvement of loss function at the first communication round, and then extend to the rest communication rounds.

\textit{1) Improvement in the first communication round:} The ML model is updated based on initialization $\mathbf{w}_{0}$ over the new sensed dataset $\mathcal{D}_1$ in the current round. Therefore, applying the Lipschitz-smooth property and Cauchy–Schwarz inequality, $C_2$ is bounded by
\begin{equation}
	\begin{split}
		C_2 & \leq \eta \sum_{n=1}^{N} \rho^{n} \left\| \nabla F(\mathbf{w}_{t-1};\mathcal{D}_{1})-  \nabla F\left( \mathbf{w}^{n}_{t-1};\mathcal{D}^{n}_{1} \right) \right\|^2 \\
		 & \leq L^2 \eta \sum_{n=1}^{N} \rho^{n} \left\| \mathbf{w}^{n}_{t-1} - \mathbf{w}^{n}_{t} \right\|^2 \\
		 & \leq L^2 \eta^3 \sum_{n=1}^{N} \rho^{n} \left\| \nabla F\left( \mathbf{w}^{n}_{t-1};\mathcal{D}^{n}_{1} \right) \right\|^2 \\
		 & \leq L^2  \alpha^2 \eta^3 \left\| \nabla F\left( \mathbf{w}^{}_{t-1};\mathcal{D}^{}_{1} \right) \right\|^2 + L^2 \eta^3 \beta^2.
	\end{split}
\end{equation}

As a result, when $\eta \leq \frac{1}{2L}$, \eqref{Lsmooth3_1} is reformulated as 
\begin{equation} \label{Lsmooth3_1_1}
    \begin{split}
        &F\left(\mathbf{w}_{t}; \mathcal{S}_{t}\right) \!-\! F\left(\mathbf{w}_{t-1};\mathcal{S}_{t-1}\right)  \leq\\
        &- \frac{\eta }{2} \left( 1 - 2L^2\eta^2 \alpha^2  \right)  \left\| \nabla F(\mathbf{w}_{0};\mathcal{D}_{1}) \right\|^2  + \eta \left\|  \boldsymbol{\varepsilon}_t  \right\|^2 +  + L^2 \eta^3   \beta^2  \\
        &\leq - \frac{\eta }{4} \left( 2 - \alpha^2  \right)  \left\| \nabla F(\mathbf{w}_{0};\mathcal{D}_{1}) \right\|^2  + \eta \left\|  \boldsymbol{\varepsilon}_t  \right\|^2 +  \frac{ \eta \beta^2 }{4}   .
    \end{split}
\end{equation}

%

\textit{2) Improvement in the rest communication rounds:} For the rest communication rounds, the ML model is updated based on both the accumulative dataset $\mathcal{S}_{t-1}$ and the newly sensed dataset $\mathcal{D}_{t}$. According to \textbf{Lemma 1 }, $C_2$ is expressed as 
\begin{small}
\begin{equation}\label{C_2}
    \begin{split}
        & C_2 \leq   \eta  \left\| \nabla F(\mathbf{w}_{t-1};\mathcal{S}_{t-1})  -  \frac{S_{t-1}}{S_{t}} \sum_{n=1}^{N}\Bar{\rho}^{n} \nabla F(\mathbf{w}^{n}_{t-1}; \mathcal{S}^{n}_{t-1}) \right. \\
		& \left. - \frac{D_t}{S_{t}} \sum_{n=1}^{N} \Tilde{\rho}^{n} \nabla F(\mathbf{w}^{n}_{t-1}; \mathcal{D}^{n}_t) \right\|^2\\
        & \! \! \leq   2 L^2 \eta \left( \frac{S^{2}_{t-1} }{S^{2}_{t}}  \sum_{n=1}^{N}\Bar{\rho}^{n} \left\|  \mathbf{w}_{t-1} \! \! - \Bar{ \mathbf{w}}^{n}_{t}  \right\|^2 + \frac{  D^{2}_t}{S^{2}_{t}} \sum_{n=1}^{N} \Tilde{\rho}^{n} \left\|  \mathbf{w}_{t-1}  - \Tilde{\mathbf{w}}^{n}_{t}  \right\|^2 \right)  \\
        & \leq 2 \alpha^2 L^2 \eta^3  \left(  \frac{ S^{2}_{t-1} }{S^{2}_{t}}  \left\| \nabla F(\mathbf{w}_{t-1};\mathcal{S}_{t-1})  \right\|^2 +   \frac{  D^{2}_t}{S^{2}_{t}} \left\| \nabla F(\mathbf{w}_{t-1};\mathcal{D}_{t})  \right\|^2 \right)  \\
        &+ 2 \beta^2  L^2 \eta^3  \left(  \frac{ S^{2}_{t-1} }{S^{2}_{t}}   +  \frac{  D^{2}_t}{S^{2}_{t}}  \right) .
    \end{split}
\end{equation}
\end{small}
where $\Bar{ \mathbf{w}}^{n}_{t} $ and $\Tilde{\mathbf{w}}^{n}_{t}$ are the updated models based on the datasets $\mathcal{S}^{n}_{t-1}$ and $\mathcal{D}_t$, respectively.

As a result, when $ \eta \leq \frac{1}{2\sqrt{2}L} \frac{S_{t}}{S_{t-1}}$, we have 
\begin{equation}\label{Lsmooth3_1_2}
    \begin{split}
        &F\left(\mathbf{w}_{t}; \mathcal{S}_{t}\right) - F\left(\mathbf{w}_{t-1};\mathcal{S}_{t-1}\right)  \leq  \\
        &\leq  - \frac{\eta }{2} \left( 1 -\frac{\alpha^2}{2} \right)  \left\| \nabla F(\mathbf{w}_{t-1};\mathcal{S}_{t-1}) \right\|^2  + \eta \left\|  \boldsymbol{\varepsilon}_t  \right\|^2 \\ 
        &+ \frac{ \eta}{4} \left( 1 - \frac{2S^{}_{t-1} D^{}_{t}}{S^{2}_{t}} \right) \frac{S^{2}_{t}}{S^{2}_{t-1}} \beta^2 + \frac{\eta \alpha^2  }{4 } \frac{D^{2}_{t}}{S^{2}_{t-1}} \left\| \nabla F(\mathbf{w}_{t-1};\mathcal{D}_{t})  \right\|^2 \\
         &\leq - \frac{\eta }{4} \left( 2 - \alpha^2  \right)  \left\| \nabla F(\mathbf{w}_{t-1};\mathcal{S}_{t-1}) \right\|^2  + \eta \left\|  \boldsymbol{\varepsilon}_t  \right\|^2  \\
         &+ \frac{ \eta}{4} \frac{S^{2}_{t}}{S^{2}_{t-1}} \beta^2 + \frac{\eta \alpha^2  }{4 } \frac{D^{2}_{t}}{S^{2}_{t-1}}  G_{t}. 
    \end{split}
\end{equation}
This ends the proof.

%
%

\ifCLASSOPTIONcaptionsoff
  \newpage
\fi

\end{document}